\documentclass{article}
\usepackage{ijcai25}

\usepackage[normalem]{ulem}

\usepackage{times}
\usepackage{soul}
\usepackage{url}
\usepackage[hidelinks]{hyperref}
\usepackage[utf8]{inputenc}
\usepackage[small]{caption}
\usepackage{graphicx}
\usepackage{amsmath}
\usepackage{amsthm}
\usepackage{booktabs}
\usepackage{algorithm}
\usepackage{algorithmic}
\usepackage[switch]{lineno}

\usepackage{amsmath,amssymb,amsfonts}
\usepackage[table]{xcolor}
\usepackage{xcolor}
\usepackage{tabularx}
\usepackage{multirow}
\usepackage{caption}
\usepackage{subcaption}
\usepackage{bm}
\usepackage{color}
\usepackage{afterpage}
\newcommand\scalemath[2]{\scalebox{#1}{\mbox{\ensuremath{\displaystyle #2}}}}
\usepackage{array}
\usepackage{bbding}
\usepackage{makecell}
\usepackage{amsthm}
\usepackage[T1]{fontenc}
\usepackage{rotating}
\usepackage{float}
\usepackage{enumitem} 
\usepackage[flushleft]{threeparttable}
\usepackage{adjustbox}
\usepackage{ragged2e}
\usepackage{environ}

\theoremstyle{definition}
\newtheorem{proposition}{Proposition}

\theoremstyle{definition}

\theoremstyle{lemma}
\newtheorem{lemma}{Lemma}

\theoremstyle{corollary}
\newtheorem{corollary}{Corollary}

\newcommand{\mymodel}{\textsc{RelGCL}}

\newtheorem{theorem}{Theorem}

\title{Rethinking Graph Contrastive Learning through Relative Similarity Preservation}

\author{
Zhiyuan Ning$^{1,2}$
\and
Pengfei Wang$^{1,2}$
\and
Ziyue Qiao$^{4}$
\and
Pengyang Wang$^{5}$\thanks{Corresponding author.}
\and 
Yuanchun Zhou$^{1,2,3}$\\
\affiliations
$^1$Computer Network Information Center, Chinese Academy of Sciences\\
$^2$University of Chinese Academy of Sciences 
$^3$Hangzhou Institute for Advanced Study, UCAS\\
$^4$School of Computing and Information Technology, Great Bay University\\
$^5$Department of Computer and Information Science, IOTSC, University of Macau 
\emails
ningzhiyuan@cnic.cn,
pfwang@cnic.cn,
zyqiao@gbu.edu.cn,
pywang@um.edu.mo,
zyc@cnic.cn
}

\begin{document}

\maketitle

\begin{abstract}
Graph contrastive learning (GCL) has achieved remarkable success by following the computer vision paradigm of preserving absolute similarity between augmented views.
However, this approach faces fundamental challenges in graphs due to their discrete, non-Euclidean nature---view generation often breaks semantic validity and similarity verification becomes unreliable.
Through analyzing 11 real-world graphs, we discover a universal pattern transcending the homophily-heterophily dichotomy: label consistency systematically diminishes as structural distance increases, manifesting as smooth decay in homophily graphs and oscillatory decay in heterophily graphs.
We establish theoretical guarantees for this pattern through random walk theory, proving label distribution convergence and characterizing the mechanisms behind different decay behaviors.
This discovery reveals that graphs naturally encode relative similarity patterns, where structurally closer nodes exhibit collectively stronger semantic relationships.
Leveraging this insight, we propose \mymodel{}, a novel GCL framework with complementary pairwise and listwise implementations that preserve these inherent patterns through collective similarity objectives.
Extensive experiments demonstrate that our method consistently outperforms 20 existing approaches across both homophily and heterophily graphs, validating the effectiveness of leveraging natural relative similarity over artificial absolute similarity.
\end{abstract}
\section{Introduction}\label{sec:intro}
\looseness=-1
Graph contrastive learning (GCL) has emerged as a powerful approach for graph self-supervised learning, demonstrating strong performance across node classification, clustering~\cite{ning2025deep,xu2024sccdcg}, and similarity search tasks~\cite{Velickovic2019DeepGI,Zhu2020DeepGC,thakoor2021large}. 
Following the success in computer vision~\cite{he2020momentum,chen2020simple}, these methods typically generate multiple views of the same graph through augmentation techniques, aiming to maximize the agreement between different views of the same instance.
This approach implicitly assumes that different views should maintain absolute similarity in the embedding space---an assumption that has proved powerful for visual representations~\cite{chen2021exploring} (as shown in Figure~\ref{fig:example} (a)).
However, graphs fundamentally differ from images in their discrete, non-Euclidean nature~\cite{thakoor2021large}. 
This fundamental difference creates two critical challenges. 
First, view generation often breaks semantic validity (as shown in Figure~\ref{fig:example} (b))---while rotating an image preserves its meaning, removing an edge from a molecular graph could yield an entirely different chemical compound with distinct properties~\cite{Lee2022AugmentationFreeSL}. 
Second, and more fundamentally, similarity verification becomes unreliable (as shown in Figure~\ref{fig:example} (c))---while humans can easily verify if two image views represent the same object, judging similarity between graph views often exceeds human intuition~\cite{Hou2022GraphMAESM}, especially for abstract graphs representing complex relationships.

\begin{figure}[!t]
\centering
\vspace{-2mm}
\includegraphics[width=0.90\linewidth]{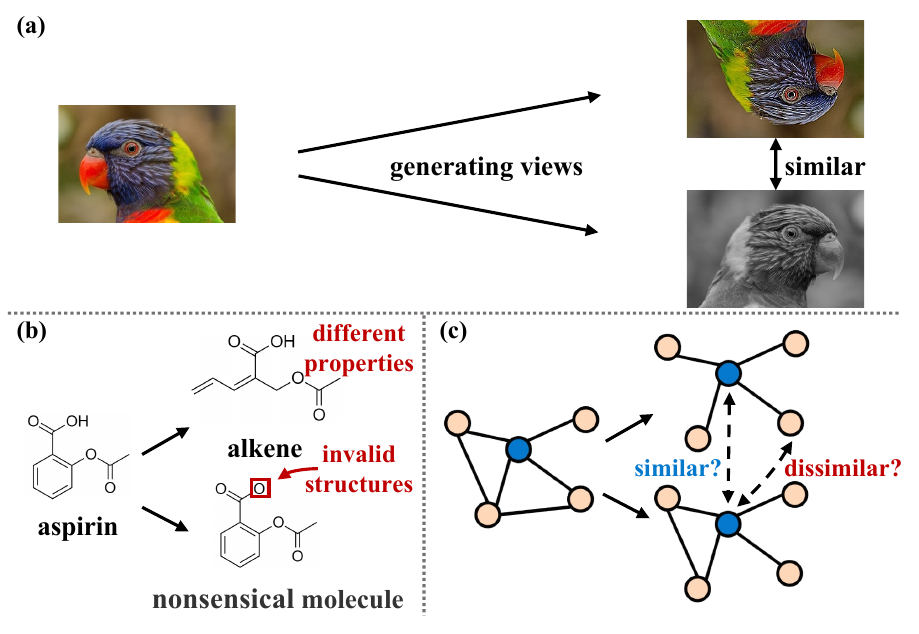}
\vspace{-4mm}
\caption{
Visual vs. Graph contrastive learning: (a) image views preserve semantics, (b) graph augmentation may alter properties, and (c) graph view similarity is hard to assess.
}
\label{fig:example}
\vspace{-6mm}
\end{figure}

\looseness=-1
These challenges suggest that enforcing absolute similarity through artificial views fundamentally misaligns with the nature of graph data. Rather than imposing external similarity constraints, a more promising direction would be to understand and leverage the inherent structural patterns that naturally exist in graphs~\cite{bramoulle2012homophily,zhu2020beyond}. 
This leads us to a more fundamental question: 
\textit{\textbf{how is similarity inherently encoded in graph structures themselves?}}
While traditional similarity analysis in graphs focuses on immediate neighborhood relationships---characterizing graphs as either homophily graphs~\cite{mcpherson2001birds} where similar nodes connect, or heterophily graphs~\cite{Pei2020GeomGCNGG} where dissimilar nodes connect---such \textit{\textbf{local}} characterization~\cite{ning2021lightcake,qiao2020context} fails to capture \textit{\textbf{broader}} similarity patterns that could provide richer signals for representation learning.

To obtain this broader perspective, we examine how node labels distribute across multi-hop neighborhoods in 11 real-world graph datasets: 8 homophily graphs and 3 heterophily graphs (as shown in Figure~\ref{fig:motivation}). 
Through quantifying semantic relationships via ``label consistency"---the average proportion of same-labeled nodes at each structural distance---we discover an unexpected universal pattern: 
despite their distinct local connectivity patterns, 
\textit{\textbf{both types of graphs show systematic diminishing of label consistency as structural distance increases.}}
While this decay manifests differently (\ul{smooth} in homophily graphs versus 
\uwave{oscillatory} in heterophily graphs), it reveals a fundamental principle: nodes that are structurally closer tend to share stronger semantic relationships collectively.
We rigorously validate this empirical observation through random walk theory~\cite{lovasz1993random,levin2017markov}, establishing the first theoretical guarantees on universal label distribution convergence and revealing the underlying mechanisms of distinct decay patterns in homophily versus heterophily graphs.
Our discovery fundamentally shifts how we define similarity in GCL---moving away from artificially imposed \textit{\textbf{absolute similarity}} (``whether two instances are similar") towards \textit{\textbf{relative similarity}} (``which instance is more similar to the anchor") inherent in structural proximity.
Based on these theoretical insights, we develop \mymodel{} (\textsc{Rel}ative \textsc{G}raph \textsc{C}ontrastive \textsc{L}earning), a novel GCL framework that leverages natural relative similarity patterns encoded in structural proximity.
Through carefully designed collective similarity objectives, we propose two complementary implementations of \mymodel{} (\mymodel{}\textsubscript{\textsc{Pair}} and \mymodel{}\textsubscript{\textsc{List}}) that preserve these inherent patterns from different perspectives.
Extensive experiments demonstrate state-of-the-art performance across 8 homophily and 3 heterophily graphs, consistently outperforming 20 existing approaches while showing strong generalization to large-scale graphs and diverse tasks. 

\begin{figure}[!t]
\centering
\vspace{-2mm}
\includegraphics[width=0.95\linewidth]{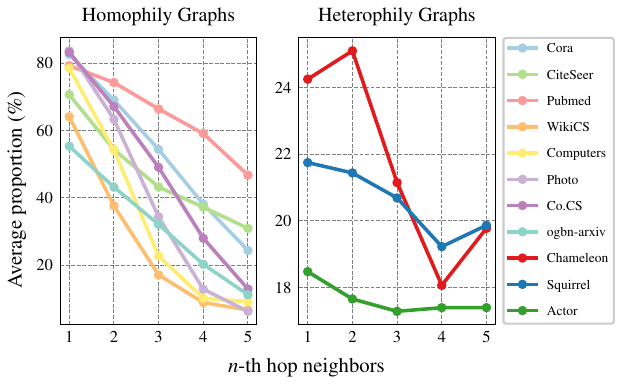}
\vspace{-4mm}
\caption{
Label consistency (the average proportion of neighbors sharing the same label as the anchor node) decay patterns: \ul{smooth monotonic decay} in homophily graphs versus \uwave{oscillatory decay} in heterophily graphs, both exhibiting \textbf{\textit{an overall diminishing trend}}.
}
\label{fig:motivation}
\vspace{-5mm}
\end{figure}

\section{Universal Decay Patterns in Graphs}
\label{sec:moti}

Having identified the limitations of imposing absolute similarity through artificial views, we now investigate how similarity is naturally encoded in graph structures. 
While traditional similarity analysis of graph focuses on local neighborhood patterns, our investigation extends to a broader neighborhood range, revealing a universal property that goes beyond local homophily-heterophily differences.

In this section, we first present empirical evidence for this universal pattern, followed by theoretical analysis that explains its underlying mechanisms.
Before the subsequent analysis, let's first define the basic notation.
Let $\mathcal{G}=\{\mathcal{V}, \mathcal{E}, \mathbf{X}, \mathbf{A}\}$ denote a graph with $N$ nodes, where $\mathcal{V}=\{v_i\}_{i=1}^N$ is the node set, $\mathcal{E} \subseteq \mathcal{V} \times \mathcal{V}$ is the edge set, $\mathbf{X}\in \mathbb{R}^{N \times D}$ is the node feature matrix, and $\mathbf{A} \in \mathbb{R}^{N \times N}$ is the adjacency matrix. For each node $v_i$, let $\mathcal{N}(v_i)^{[n]}$ denote its $n$-th hop neighbors, and $y_i$ denote the label of node $v_i$. When considering a $k$-hop neighborhood range, we denote the set of neighbors at different hops as $\mathbb{N}(v_i)^{[k]}=\{\mathcal{N}(v_i)^{[1]}, \mathcal{N}(v_i)^{[2]}, \dots, \mathcal{N}(v_i)^{[k]}, \mathcal{N}(v_i)^{[k+1]}\}$, where $\mathcal{N}(v_i)^{[k+1]}$ contains all nodes beyond $k$ hops.

\subsection{Empirical Evidence of Universal Decay}
\label{sec:empirical}  

We examine a collection of 11 real-world datasets: 8 homophily graphs and 3 heterophily graphs.
To quantify semantic relationships at different structural scales, we introduce the empirical label consistency measure ($\mathit{LC}_{\text{emp}}$). For a given structural distance $k$, $\mathit{LC}_{\text{emp}}(k)$ captures the average proportion of $k$-hop neighbors sharing the same label as the anchor node:

\begin{equation}
\scalemath{0.90}{
\begin{aligned}
\mathit{LC}_{\text{emp}}(k)=\frac{1}{N} \sum_{v_i \in \mathcal{V}} \frac{|\{v_j: v_j \in \mathcal{N}(v_i)^{[k]} \wedge y_i=y_j\}|}{|\mathcal{N}(v_i)^{[k]}|}.
\end{aligned}}
\label{eq:lc_stat}
\end{equation}

Our analysis of $\mathit{LC}_{\text{emp}}(k)$ across different structural distances reveals how label consistency evolves in different types of graphs (as shown in Figure~\ref{fig:motivation}):

\begin{itemize}[left=0pt]
    \item \textbf{Homophily Graphs:} Display \ul{smooth, monotonic decay} in label consistency, reflecting their local preference for similar neighbors. Starting from high initial values ($\mathit{LC}_{\text{emp}}(1)$ typically > 0.5), the consistency gradually diminishes as structural distance increases.
    
    \item \textbf{Heterophily Graphs:} Exhibit \uwave{oscillatory decay patterns}, where label consistency may occasionally increase at certain hops due to the graph's tendency to connect dissimilar neighbors, though the overall trend remains downward. Starting from low initial values ($\mathit{LC}_{\text{emp}}(1)$ typically < 0.5), the pattern shows possible local increases but maintains a decreasing trend with structural distance.
\end{itemize}

Despite these distinct decay patterns, a universal property emerges: label consistency exhibits an overall diminishing trend with increasing structural distance across both graph types. 
This decay pattern is statistical in nature. At each hop, we can still find nodes sharing the same label as the anchor node, but the proportion of same-labeled nodes among all neighbors at each hop, when averaged across all nodes in the graph, exhibits a clear decreasing trend as structural distance increases.
Such universal decay suggests a fundamental connection between structural proximity and semantic similarity in graphs, independent of whether the graph exhibits homophily or heterophily.

\subsection{Random Walk Theory of Label Consistency}
\label{sec:theory}

While these empirical observations reveal an intriguing universal pattern transcending local connectivity differences, they raise fundamental questions: Why does label consistency universally decay with distance? What mechanisms drive the distinct decay patterns in homophily versus heterophily graphs? 
To answer these questions, we turn to random walk theory~\cite{lovasz1993random,levin2017markov}, which provides a principled framework for analyzing how information propagates through graph structure.

Following the notations defined in Section~\ref{sec:moti}, consider a random walk starting from a node with label $i$ in a graph with $c$ distinct labels. 
We analyze this process at two levels: the microscopic node-level transitions, where a walker moves between individual nodes with probability $P_{uv} = A_{uv}/\text{deg}(u)$, and the macroscopic label-level dynamics, where we track transitions between different label classes. At the label level, let $p_k(j|i)$ denote the probability of reaching a node with label $j$ after $k$ steps from any node with label $i$. 
The theoretical label consistency $\mathit{LC}_{\text{prob}}(k) = p_k(i|i)$ then measures the probability of returning to a node with the same label after $k$ steps. 
Let $l(v)$ denote the label of node $v$. The evolution of these probabilities is governed by a label transition matrix $T$, where the transition probability $T_{ij}$ is defined in Appendix~\ref{appendix:theory:random_walk} along with detailed discussions on the random walk framework.
We establish the following fundamental results about label distribution dynamics:

\begin{theorem}[Label Distribution Convergence]
\label{thm:convergence}
For a connected graph where each node has a self-loop, there exists a unique probability distribution $\pi$ (where $\pi_j = \frac{\sum_{u:l(u)=j}\text{deg}(u)}{\sum_v \text{deg}(v)}$) such that:
\begin{equation}
\lim_{k \to \infty} p_k(j|i) = \pi_j,
\end{equation}
moreover, the convergence is exponential:
\begin{equation}
|p_k(j|i) - \pi_j| \leq C\lambda^k,
\end{equation}
where $C > 0$ and $\lambda < 1$ are constants determined by the graph structure.
\end{theorem}

\begin{proof}[Proof Sketch]
The key is showing that $T$ satisfies essential Markov chain properties. The Perron-Frobenius theorem then yields the unique stationary distribution $\pi$ and exponential convergence. Full proof in Appendix~\ref{appendix:theory:proof_convergence}.
\end{proof}

This stationary distribution $\pi_j$ represents the proportion of edges connected to nodes with label $j$, naturally characterizing the structural importance of different labels (see Appendix~\ref{appendix:theory:proof_convergence} for detailed interpretation). This convergence result leads us to analyze how Label Consistency evolves during this process:

\begin{corollary}[Label Consistency Decay]
\label{cor:lc_decay}
The Label Consistency $\mathit{LC}_{\text{prob}}(k) = p_k(i|i)$ exhibits a decay pattern where:
\begin{equation}
\scalemath{0.90}{
\begin{aligned}
&\mathit{LC}_{\text{prob}}(0) = 1, \\
&\lim_{k \to \infty} \mathit{LC}_{\text{prob}}(k) = \pi_i < 1, \\
&|\mathit{LC}_{\text{prob}}(k) - \pi_i| \leq C\lambda^k.
\end{aligned}}
\end{equation}
\end{corollary}

\begin{proof}[Proof Sketch]
Graph connectivity ensures $\pi_i<1$ since some edges must connect to nodes with other labels. Exponential decay follows from spectral decomposition of $T^k$. Detailed proof in Appendix~\ref{appendix:theory:proof_decay}.
\end{proof}

To further explain the distinct decay patterns in homophily versus heterophily graphs, we analyze the spectral properties of their transition matrices in a simplified two-label setting:

\begin{proposition}[Decay Pattern Characterization]
\label{prop:decay_patterns}
In a two-label simplified model (see Appendix~\ref{appendix:theory:contrasting} for assumptions):
\begin{itemize}[left=0pt]
    \item \textbf{Homophily Graphs:} $T \approx \begin{bmatrix}p & 1-p \\ 1-p & p\end{bmatrix}$ with $p \gg 0.5$, leading to $\lambda_2 = 2p-1 > 0$
    \item \textbf{Heterophily Graphs:} $T \approx \begin{bmatrix}1-p & p \\ p & 1-p\end{bmatrix}$ with $p \gg 0.5$, resulting in $\lambda_2 = 1-2p < 0$
\end{itemize}
These eigenvalue patterns explain the distinct decay behaviors: monotonic decay when $\lambda_2 > 0$ (homophily) versus oscillatory decay when $\lambda_2 < 0$ (heterophily). A rigorous analysis of these patterns can be found in Appendix~\ref{appendix:theory:contrasting}.
\end{proposition}

While we analyze a binary setting for analytical clarity, this simplified model illuminates how structural bias in transition probabilities (dominated by self-transitions in homophily vs. cross-label transitions in heterophily) determines the sign of $\lambda_2$ and shapes decay patterns. This analysis provides an interpretable lens into how local properties influence global behaviors, while our general framework (Theorem~\ref{thm:convergence} and Corollary~\ref{cor:lc_decay}) establishes the universality for arbitrary number of labels.
In summary, empirical observations and theoretical analysis show that \textit{\textbf{label consistency decays with structural distance---both in smooth and oscillatory patterns}}. 
This decay law will fundamentally change how we view similarity in GCL, leading to a new framework that naturally captures structural relationships, detailed next.

\begin{figure*}[!t]
\centering
\includegraphics[width=0.90\linewidth]{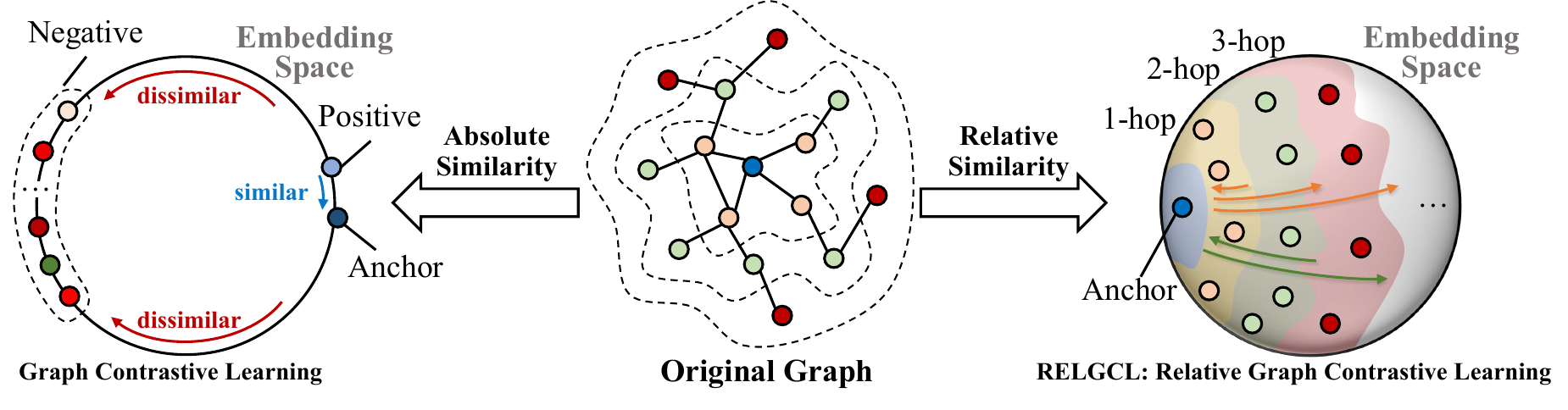}
\vspace{-4.0mm}
\caption{
A philosophical comparison of absolute similarity and relative similarity in GCL. The right side presents the core idea of \mymodel{}.
}
\vspace{-4mm}
\label{fig:framework}
\end{figure*}
\section{Methodology}

Building upon our established observations and theoretical foundations of the universal label consistency diminishing property, we propose \mymodel{} (\textsc{Rel}ative \textsc{G}raph \textsc{C}ontrastive \textsc{L}earning), a principled framework to leverage this inherent characteristic for GCL.
The key insight is to preserve the naturally encoded relative similarity patterns that persist across both homophily and heterophily graphs.
In this section, our goal is to learn a graph encoder (we take graph convolutional network~\cite{welling2016semi}) $f: (\mathbf{X}, \mathbf{A}) \rightarrow \mathbf{H}\in \mathbb{R}^{N \times d}$ that maps nodes to low-dimensional representations ($d \ll D$) while preserving the inherent relative similarity patterns in graphs. 
Let $\mathbf{h}_i \in \mathbb{R}^d$ denote the learned representation of node $v_i$ in the embedding space (i.e., the $i$-th row of $\mathbf{H}$). For clarity, we denote $\mathbb{H}_i^{[n]} = \{\mathbf{h}_j \in \mathbb{R}^d: v_j \in \mathcal{N}(v_i)^{[n]}\}$ as the set of representations in the embedding space corresponding to $v_i$'s $n$-th hop neighbors.

\subsection{From Label Consistency to Relative Similarity}
\label{sec:relative_similarity}

Building upon the universal label consistency diminishing property established in Section~\ref{sec:moti}, we formalize this statistical pattern as a principled foundation for defining relative similarity in graphs.
Formally, for any node $v_i$, we can quantify the label consistency of its $n$-th hop neighbors through:
\begin{equation}
\scalemath{0.90}{
\begin{aligned}
    \text{sim}_{\text{stat}}(v_i,n) = \frac{|\{v_j: v_j \in \mathcal{N}(v_i)^{[n]} \wedge y_i=y_j\}|}{|\mathcal{N}(v_i)^{[n]}|},
\end{aligned}}
\end{equation}
similarly, for nodes beyond hop $n$, we define:
\begin{equation}
\scalemath{0.85}{
\begin{aligned}
    \text{sim}_{\text{stat}}(v_i,{>}n) = \frac{|\{v_j: v_j \in \bigcup_{m>n} \mathcal{N}(v_i)^{[m]} \wedge y_i=y_j\}|}{|\bigcup_{m>n} \mathcal{N}(v_i)^{[m]}|}.
\end{aligned}}
\end{equation}

Our empirical and theoretical analyses in Section~\ref{sec:moti} prove that while this relationship may not hold deterministically for every individual node, it holds statistically across the graph:
\begin{equation}
\scalemath{0.90}{
\begin{aligned}
    \mathbb{E}_{v\in\mathcal{V}}[\text{sim}_{\text{stat}}(v,n)] > \mathbb{E}_{v\in\mathcal{V}}[\text{sim}_{\text{stat}}(v,{>}n)],
\end{aligned}}
\end{equation}
this statistical property reveals a fundamental characteristic of graphs: nodes at hop $n$ collectively exhibit stronger semantic similarity to the anchor node compared to further nodes, regardless of the graph's homophily nature. We term this the ``relative similarity" property, as it describes similarity from a \textit{\textbf{relative}} and \textit{\textbf{collective}} perspective rather than an \textit{\textbf{absolute}} or \textit{\textbf{individual}} one.
To preserve this property in the learned representations $\mathbf{H}$, we need our embeddings to satisfy an analogous relationship in the representation space. Let $s(\cdot,\cdot)$ denote a similarity measure between node representations. For any node $v_i$, its representation $\mathbf{h}_i$ should maintain:
\begin{equation}
\scalemath{0.85}{
\begin{aligned}
    \mathbb{E}\left[\frac{\sum_{\mathbf{h}_j \in \mathbb{H}_i^{[n]}} s(\mathbf{h}_i, \mathbf{h}_j)}{|\mathbb{H}_i^{[n]}|}\right] > 
    \mathbb{E}\left[\frac{\sum_{\mathbf{h}_j \in \bigcup_{m>n} \mathbb{H}_i^{[m]}} s(\mathbf{h}_i, \mathbf{h}_j)}{|\bigcup_{m>n} \mathbb{H}_i^{[m]}|}\right],
\end{aligned}}
\end{equation}
this formulation translates the statistical patterns in label space to constraints in the embedding space, providing a principled objective for GCL. The key challenge lies in how to effectively model and preserve such collective relative similarity, which we address in the following sections.

\subsection{Building Collective Similarity}

The relative similarity property established in Section~\ref{sec:relative_similarity} reveals two key characteristics that guide our modeling choices: 
\textbf{(1)} similarity should be measured collectively across groups of nodes rather than individual pairs, 
and \textbf{(2)} the relationship is statistical rather than deterministic. These insights drive us to develop a framework that can effectively handle collective similarity patterns.
A straightforward approach to handle multiple positive examples~\cite{khosla2020supervised} in contrastive learning is to extend InfoNCE loss~\cite{oord2018representation} by summing over positive examples:
\begin{equation}
\scalemath{0.85}{
\begin{aligned}
\mathcal{L}^{\text{out}}=-\sum_{p \in \mathbf{P}} \log \frac{\exp(\theta(q, p)/\tau)}{\exp(\theta(q, p)/\tau)+\sum_{n \in \mathbf{N}} \exp(\theta(q, n)/\tau)},
\end{aligned}}
\label{equation:out}
\end{equation}
where $\theta(\cdot,\cdot)=s(g(\cdot), g(\cdot))$ combines a nonlinear projection head $g(\cdot)$ and a cosine similarity measure $s(\cdot,\cdot)$, and $\tau$ is a temperature parameter~\cite{Zhu2021GraphCL,chen2020simple,tschannen2019mutual}. However, this formulation enforces high similarity between the query $q$ and each individual positive example $p$~\cite{hoffmann2022ranking,khosla2020supervised}, conflicting with the statistical nature of relative similarity in graphs.
A more principled approach is to sum over positive examples inside the logarithm~\cite{miech2020end}:
\begin{equation}
\scalemath{0.85}{
\begin{aligned}
\mathcal{L}^{\text{in}}=-\log \frac{\sum_{p \in \mathbf{P}} \exp(\theta(q, p)/\tau)}{\sum_{p \in \mathbf{P}} \exp(\theta(q, p)/\tau)+\sum_{n \in \mathbf{N}} \exp(\theta(q, n)/\tau)},
\end{aligned}}
\label{equation:in}
\end{equation}
this formulation offers several advantages: 
\textbf{(1)} The summation inside the logarithm naturally handles collective patterns by comparing aggregate similarities between groups~\cite{hoffmann2022ranking,miech2020end}, 
\textbf{(2)} it allows for variation within groups while maintaining their overall relative relationships, 
and \textbf{(3)} the soft nature of the exponential terms aligns with the statistical nature of our similarity definition. 
From an optimization perspective, minimizing $\mathcal{L}^{\text{in}}$ leads to maximizing the similarity ratio $\frac{\sum_{p \in \mathbf{P}} \exp(\theta(q, p)/\tau)}{\sum_{n \in \mathbf{N}} \exp(\theta(q, n)/\tau)}$, which simultaneously increases collective similarity with positive examples $\mathbf{P}$ while decreasing that with negative examples $\mathbf{N}$---naturally aligning with our goal of modeling relative similarity, as it enhances the similarity distinction between different node groups while maintaining their collective nature.

\subsection{Learning Framework}
\label{sec:learning_framework}

Based on the collective similarity building block, specifically adopting the $\mathcal{L}^{\text{in}}$ formulation (as shown in Equation~\ref{equation:in}) due to its advantages in handling collective patterns, we now present \mymodel{} for preserving the graph-inherent relative similarity (an overview of \mymodel{} is presented in Figure~\ref{fig:framework}). 
For each node $v_i$, within a $k$-hop neighborhood range, we aim to ensure that its $n$-th hop neighbors ($n \leq k$) are collectively more similar to it than further neighbors. 
We propose two complementary implementations of \mymodel{} that differ in how they handle the comparison with further neighbors.
Both approaches use a threshold $\alpha \in (0,1)$ to prevent over-optimization of similarity ratios, reflecting the statistical rather than deterministic nature of relative similarity established in Section~\ref{sec:theory}---while closer nodes exhibit stronger collective similarity statistically, this relationship should not be enforced as an absolute constraint.

\noindent\textbf{\mymodel{}\textsubscript{\textsc{pair}}.} This approach examines relative similarity through \textit{\textbf{sequential pairwise comparisons between hops}} (e.g., comparing hop-1 vs. hop-2, hop-1 vs. hop-3, hop-2 vs. hop-3, etc.). For any node $v_i$, we define the pairwise similarity ratio between hop $n$ and hop $n+m$ as:

\begin{equation}
\scalemath{0.70}{
\begin{aligned}
    r_{n,m}(\mathbf{h}_i) = \frac{\sum_{\mathbf{h}_* \in \mathbb{H}_i^{[n]}}\exp(\theta(\mathbf{h}_i, \mathbf{h}_*)/\tau)}
    {\sum_{\mathbf{h}_* \in \mathbb{H}_i^{[n]}}\exp(\theta(\mathbf{h}_i, \mathbf{h}_*)/\tau) + 
    \sum_{\mathbf{h}_\diamond \in \mathbb{H}_i^{[n+m]}}\exp(\theta(\mathbf{h}_i, \mathbf{h}_\diamond)/\tau)}.
\end{aligned}}
\end{equation}

The overall pairwise objective is:
\begin{equation}
\scalemath{0.85}{
\begin{aligned}
    \mathcal{L}_{\text{pair}} = - \sum\limits_{v_i\in \mathcal{V}} \frac{1}{k} \sum\limits_{n=1}^{k} \sum\limits_{m=1}^{k-n+1} \log 
    [\min(r_{n,m}(\mathbf{h}_i), \alpha)].
\end{aligned}}
\label{equation:pair-wise}
\end{equation}

The pairwise implementation enables fine-grained modeling of structural transitions between specific hop distances, potentially making it more suitable for capturing oscillatory patterns in graphs (Proposition~\ref{prop:decay_patterns}).

\noindent\textbf{\mymodel{}\textsubscript{\textsc{list}}.} This approach compares each hop against \textit{\textbf{all its subsequent hops as a whole}} (e.g., comparing hop-1 vs. \{hop-2, hop-3, ...\}, hop-2 vs. \{hop-3, hop-4, ...\}, etc.). For node $v_i$, we define the listwise similarity ratio for hop $n$ as:

\begin{equation}
\scalemath{0.75}{
\begin{aligned}
    r_n(\mathbf{h}_i) = \frac{\sum_{\mathbf{h}_* \in \mathbb{H}_i^{[n]}}\exp(\theta(\mathbf{h}_i, \mathbf{h}_*)/\tau)}
    {\sum_{m=n}^{k+1} \sum_{\mathbf{h}_\diamond \in \mathbb{H}_i^{[m]}}\exp(\theta(\mathbf{h}_i, \mathbf{h}_\diamond)/\tau)}.
\end{aligned}}
\end{equation}

The overall listwise objective is:
\begin{equation}
\scalemath{0.85}{
\begin{aligned}
    \mathcal{L}_{\text{list}} = - \sum\limits_{v_i\in \mathcal{V}} \frac{1}{k} \sum\limits_{n=1}^{k} \log [\min(r_n(\mathbf{h}_i), \alpha)].
\end{aligned}}
\label{equation:list-wise}
\end{equation}

The listwise implementation aggregates all subsequent hops collectively, providing a holistic view of the broader neighborhood structure.

The threshold $\alpha$ in both objectives prevents the similarity ratios from being pushed to extremes while preserving the relative similarity patterns. 
Both implementations effectively preserve the relative similarity patterns while offering complementary perspectives on handling multi-hop relationships. 
A detailed comparison with previous approaches and complexity analysis is provided in Appendix~\ref{appendix:method_analysis}.

\section{Experiments}

\subsection{Experimental Setting}

A  detailed experimental setting is provided in Appendix~\ref{appendix:experimental-setting}.

\noindent\textbf{Datasets.} To validate the universality of our approach, we conduct experiments on 11 real-world datasets spanning diverse domains and scales (2K to 169K nodes), including 8 homophily graphs (Cora, Citeseer, Pubmed, WikiCS, Amazon-Computers, Amazon-Photo, Coauthor-CS, and ogbn-arxiv) and 3 heterophily graphs (Chameleon, Squirrel, and Actor). These datasets cover citation networks~\cite{sen2008collective,hu2020open}, co-occurrence networks~\cite{mcauley2015image,sinha2015overview}, and heterophily networks~\cite{Pei2020GeomGCNGG}, with homophily ratios ranging from 0.18 to 0.83.

\noindent\textbf{Baselines.} We compare \mymodel{} against 20 representative methods across 5 categories: 
2 supervised methods~\cite{welling2016semi,Velickovic2018GraphAN}, 
3 graph autoencoder methods~\cite{Kipf2016VariationalGA,Hou2022GraphMAESM}, 
9 augmentation-based GCL methods~\cite{Velickovic2019DeepGI,Zhu2020DeepGC,thakoor2021large,Zhang2021FromCC}, 
5 augmentation-free GCL methods~\cite{Peng2020GraphRL,Lee2022AugmentationFreeSL},
1 multi-task self-supervised learning based method~\cite{jin2021automated}. 
For baselines that either don't use standard splits or don't report results on certain datasets, we reproduce their results using official implementations under the same experimental setup for fair comparison.

\noindent\textbf{Evaluation Protocol.} Following standard practice in GCL~\cite{Velickovic2019DeepGI,Zhu2020DeepGC,thakoor2021large}, we evaluate using linear evaluation: first train the graph encoder in a self-supervised manner using our relative similarity objectives, then freeze it to generate node embeddings for training a logistic regression classifier. For datasets with multiple splits (e.g., WikiCS with 20 public splits), we conduct experiments on all provided splits. We report the average classification accuracy and standard deviation over 20 random runs with different random seeds to ensure reliability.

\noindent\textbf{Implementation Details.} We implement \mymodel{} using GCN~\cite{welling2016semi} as the encoder, optimized with Adam~\cite{kingma2014adam} on a NVIDIA V100 GPU. Based on our empirical analysis in Section~\ref{sec:empirical}, we set the neighborhood range $k=4$ to capture meaningful structural patterns, as the number of semantically similar neighbors becomes particularly small beyond 4 hops. For fair comparison, we use the same architecture and training procedure across all competing methods, with hyperparameters tuned on validation sets through grid search over standard ranges. The similarity threshold $\alpha$ controls the strength of relative similarity constraints as discussed in Section~\ref{sec:learning_framework}.

\begin{table*}[!t]
    \begin{minipage}[t]{0.65\linewidth}
        \centering
        \scriptsize
        \begin{threeparttable}
        \setlength{\tabcolsep}{0.15mm}
        \begin{tabular}{@{}lccccccc|c@{}}
        \toprule
        \textbf{Method} & \textbf{Cora} & \textbf{Citeseer} & \textbf{Pubmed} & \textbf{WikiCS} & \textbf{Computers} & \textbf{Photo} & \textbf{Co.CS} & \textbf{Rank} \\ 
        \midrule
        GCN & $81.5$ & $70.3$ & $79.0$ & $77.19{\pm}.12$ & $86.51{\pm}.54$ & $92.42{\pm}.22$ & $93.03{\pm}.31$ & 16.7 \\
        GAT & $83.0{\pm}.7$ & $72.5{\pm}.7$ & $79.0{\pm}.3$ & $77.65{\pm}.11$ & $86.93{\pm}.29$ & $92.56{\pm}.35$ & $92.31{\pm}.24$ & 14.4 \\
        GAE & $71.5{\pm}.4$ & $65.8{\pm}.4$ & $72.1{\pm}.5$ & $70.15{\pm}.01$ & $85.27{\pm}.19$ & $91.62{\pm}.13$ & $90.01{\pm}.71$ & 21.4 \\
        VGAE & $76.3{\pm}.2$ & $66.8{\pm}.2$ & $75.8{\pm}.4$ & $75.63{\pm}.19$ & $86.37{\pm}.21$ & $92.20{\pm}.11$ & $92.11{\pm}.09$ & 19.7 \\
        GraphMAE & $84.2{\pm}.4$ & $73.4{\pm}.4$ & $81.1{\pm}.4$ & $77.12{\pm}.30$ & $79.44{\pm}.48$ & $90.71{\pm}.40$ & $93.13{\pm}.15$ & 11.9 \\
        DGI & $82.3{\pm}.6$ & $71.8{\pm}.7$ & $76.8{\pm}.6$ & $75.35{\pm}.14$ & $83.95{\pm}.47$ & $91.61{\pm}.22$ & $92.15{\pm}.63$ & 19.0 \\
        MVGRL & $83.5{\pm}.4$ & $73.3{\pm}.5$ & $80.1{\pm}.7$ & $77.52{\pm}.08$ & $87.52{\pm}.11$ & $91.74{\pm}.07$ & $92.11{\pm}.12$ & 13.6 \\
        GRACE & $81.9{\pm}.4$ & $71.2{\pm}.5$ & $80.6{\pm}.4$ & $78.19{\pm}.01$ & $86.25{\pm}.25$ & $92.15{\pm}.24$ & $92.93{\pm}.01$ & 15.4 \\
        GCA & $82.1{\pm}.1$ & $71.3{\pm}.4$ & $80.2{\pm}.4$ & $78.30{\pm}.00$ & $87.85{\pm}.31$ & $92.49{\pm}.09$ & $93.10{\pm}.01$ & 13.7 \\
        BGRL & $82.7{\pm}.6$ & $71.1{\pm}.8$ & $79.6{\pm}.5$ & $79.31{\pm}.55$ & $89.62{\pm}.37$ & $93.07{\pm}.34$ & $92.67{\pm}.21$ & 11.7 \\
        G-BT & $81.5{\pm}.4$ & $71.9{\pm}.5$ & $80.4{\pm}.6$ & $76.65{\pm}.62$ & $88.14{\pm}.33$ & $92.63{\pm}.44$ & $92.95{\pm}.17$ & 14.3 \\
        CCA-SSG & $84.2{\pm}.4$ & $73.1{\pm}.3$ & $81.6{\pm}.4$ & $78.65{\pm}.14$ & $88.74{\pm}.28$ & $93.14{\pm}.14$ & $93.31{\pm}.22$ & 6.4 \\
        gCooL & $83.2{\pm}.5$ & $72.7{\pm}.4$ & $80.5{\pm}.4$ & $78.74{\pm}.04$ & $88.85{\pm}.14$ & $93.18{\pm}.12$ & $93.32{\pm}.02$ & 8.3 \\
        HomoGCL & $\mathbf{84.5}{\pm}.5$ & $72.3{\pm}.7$ & $81.1{\pm}.3$ & $78.84{\pm}.47$ & $88.46{\pm}.20$ & $92.92{\pm}.18$ & $92.74{\pm}.22$ & 8.9 \\
        COSTA & $83.3{\pm}.3$ & $72.1{\pm}.3$ & $81.1{\pm}.2$ & $79.12{\pm}.02$ & $88.32{\pm}.03$ & $92.56{\pm}.45$ & $92.94{\pm}.10$ & 9.7 \\
        SUGRL & $83.4{\pm}.5$ & $73.0{\pm}.4$ & $81.9{\pm}.3$ & $78.97{\pm}.22$ & $88.91{\pm}.22$ & $92.85{\pm}.24$ & $92.83{\pm}.23$ & 7.9 \\
        AFGRL & $83.2{\pm}.4$ & $72.6{\pm}.3$ & $80.8{\pm}.6$ & $77.62{\pm}.49$ & $89.88{\pm}.33$ & $93.22{\pm}.28$ & $93.27{\pm}.17$ & 8.0 \\
        SP-GCL & $83.2{\pm}.2$ & $72.0{\pm}.4$ & $79.1{\pm}.8$ & $79.01{\pm}.51$ & $89.68{\pm}.19$ & $92.49{\pm}.31$ & $91.92{\pm}.10$ & 12.6 \\
        GraphACL & $84.2{\pm}.3$ & $73.6{\pm}.2$ & $82.0{\pm}.2$ & $79.27{\pm}.45$ & $89.80{\pm}.25$ & $93.31{\pm}.19$ & $92.77{\pm}.14$ & \underline{5.3} \\
        AUTOSSL & $83.1{\pm}.4$ & $72.1{\pm}.4$ & $80.9{\pm}.6$ & $76.80{\pm}.13$ & $88.18{\pm}.43$ & $92.71{\pm}.32$ & $93.35{\pm}.09$ & 11.0 \\
        \midrule
        \rowcolor{gray!20} \mymodel{}\textsubscript{\textsc{Pair}} & \underline{$84.4{\pm}.2$} & $\mathbf{73.7{\pm}.4}$ & \underline{$82.2{\pm}.3$} & \underline{$80.14{\pm}.51$} & $\mathbf{90.14{\pm}.35}$ & $\mathbf{93.42{\pm}.44}$ & $\mathbf{93.53{\pm}.12}$ & $\mathbf{1.6}$ \\
        \rowcolor{gray!20} \mymodel{}\textsubscript{\textsc{List}} & $\mathbf{84.5{\pm}.4}$ & \underline{$73.6{\pm}.5$} & $\mathbf{82.7{\pm}.4}$ & $\mathbf{80.16{\pm}.58}$ & \underline{$89.99{\pm}.38$} & \underline{$93.40{\pm}.48$} & \underline{$93.50{\pm}.12$} & $\mathbf{1.6}$ \\
        \bottomrule
        \end{tabular}
        \end{threeparttable}
    \caption{Classification accuracies on 7 GCL benchmark datasets. `Rank' refers to the average ranking across datasets. \textbf{Bold} indicates the best and \underline{underline} indicates the runner-up.}
    \label{tab:main_results}
    \end{minipage}%
    \raisebox{-17.7mm}{
    \begin{minipage}[t]{0.33\linewidth}
        \centering
        \scriptsize
        \setlength{\tabcolsep}{0.10mm}{
        \begin{tabular}{@{}lccc@{}}
        \toprule
\textbf{Method} & \textbf{Chameleon} & \textbf{Squirrel} & \textbf{Actor} \\ \midrule
    DGI            &       $60.27 \pm 0.70$             &       $42.22 \pm 0.63$            &      $28.30 \pm 0.76$          \\
    GRACE            &       $61.24 \pm 0.53$             &      $41.09 \pm 0.85$             &       $28.27 \pm 0.43$         \\
    GCA            &        $60.94 \pm 0.81$            &      $41.53 \pm 1.09$             &      $28.89 \pm 0.50$          \\
    BGRL            &        $64.86 \pm 0.63$            &      $46.24 \pm 0.70$             &      $28.80 \pm 0.54$          \\
    AFGRL            &        $59.03 \pm 0.78$            &      $42.36 \pm 0.40$             &      $27.43 \pm 1.31$          \\
    SP-GCL            &      $65.28 \pm 0.53$              &      $52.10 \pm 0.67$             &     $28.94 \pm 0.69$           \\ 
    GraphACL            &        \underline{$69.12\pm0.24$}           &        $54.05\pm0.13$            &       $30.03\pm0.13$         \\
    \midrule
    \rowcolor{gray!20} \mymodel{}\textsubscript{\textsc{Pair}}            &        $\mathbf{69.25 \pm 0.89}$            &      $\mathbf{57.67 \pm 0.96}$             &         $\mathbf{30.32 \pm 0.91}$       \\
    \rowcolor{gray!20} \mymodel{}\textsubscript{\textsc{List}}            &         $69.06 \pm 0.86$           &       \underline{$57.53 \pm 0.92$}            &        \underline{$30.23 \pm 0.88$}        \\
 \bottomrule
        \end{tabular}}
        \vspace{-2mm}  
        \caption{Performance on 3 heterophily graphs.}
        \label{tab:heter}
        \vspace{0.5mm}   
        
        \setlength{\tabcolsep}{0.24mm}{
        \begin{tabular}{@{}lccc@{}}
        \toprule
\textbf{Method} & \textbf{Validation} & \textbf{Test}  \\ \midrule
    DGI            &       $71.26 \pm 0.11$             &       $70.34 \pm 0.16$                     \\
    GRACE-Sampling            &       $72.61 \pm 0.15$             &       $71.51 \pm 0.11$                  \\
    G-BT            &       $71.16 \pm 0.14$             &      $70.12 \pm 0.18$                    \\
    BGRL            &      $72.53 \pm 0.09$              &       $71.64 \pm 0.12$                     \\ 
    CCA-SSG            &      $72.34 \pm 0.21$              &       $71.24 \pm 0.20$                     \\ 
    GraphACL            &          $72.59\pm0.20$          &              $71.68\pm0.22$              \\     
    \midrule
    \rowcolor{gray!20} \mymodel{}\textsubscript{\textsc{Pair}}            &        \underline{$72.69 \pm 0.14$}                       &         \underline{$72.06 \pm 0.20$}       \\ 
    \rowcolor{gray!20} \mymodel{}\textsubscript{\textsc{List}}           &         $\mathbf{72.75 \pm 0.12}$                     &        $\mathbf{72.24 \pm 0.19}$        \\
     \midrule
     Supervised GCN            &        $73.00 \pm 0.17$            &        $71.74 \pm 0.29$                   \\ 
    \bottomrule
        \end{tabular}}
        \vspace{-2mm}  
        \caption{Performance on ogbn-arxiv.}
        \label{tab:ogb}
    \end{minipage}}
\end{table*}

\begin{figure*}[!t]
\centering
\includegraphics[width=0.95\linewidth]{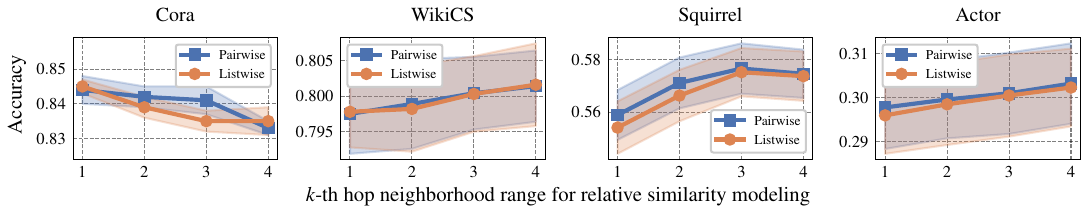}
\vspace{-4mm}  
\caption{Results of \mymodel{} with neighborhood range changing as $k=1, 2, 3, 4$, respectively.}
\vspace{-6mm}
\label{fig:different_hops_main}
\end{figure*}

\subsection{Overall Performance}

\noindent\textbf{Performance on Homophily Graphs.}
We first evaluate \mymodel{} on 7 widely-used homophily graphs. 
Table~\ref{tab:main_results} presents a comprehensive comparison against 20 baselines. Several key observations emerge:
\textbf{(1)} Both implementations of \mymodel{} achieve superior performance across datasets, outperforming all baselines with average ranks of both 1.6 versus the best baseline rank of 5.3. This demonstrates that leveraging inherent structural patterns through relative similarity is more effective than artificially imposing absolute similarity constraints through augmentations.
\textbf{(2)} The comparable performance between pairwise and listwise implementations (achieving SOTA on 4 and 3 datasets respectively) validates the robustness of our relative similarity framework---both approaches effectively capture the smooth decay patterns theoretically predicted in Section~\ref{sec:theory}.
\textbf{(3)} Methods like SUGRL~\cite{Mo2022SimpleUG}, AFGRL~\cite{Lee2022AugmentationFreeSL} and GraphACL~\cite{xiao2024simple} that only leverage local absolute similarity, or AUTOSSL~\cite{jin2021automated} that relies on homophily-guided task search, achieve inferior results. This highlights the advantage of modeling broader structural-semantic relationships over focusing solely on immediate neighborhoods or artificial task designs.

\noindent\textbf{Performance on Heterophily Graphs.}
We further evaluate \mymodel{} on 3 heterophily graphs. Table~\ref{tab:heter} demonstrates the effectiveness of our approach:
\textbf{(1)} Both implementations of \mymodel{} substantially outperform all baselines across these datasets, validating that our approach of modeling relative similarity patterns remains effective even when local connectivity patterns differ significantly from homophily graphs.
\textbf{(2)} \mymodel{}\textsubscript{\textsc{Pair}} consistently outperforms \mymodel{}\textsubscript{\textsc{List}} across all 3 heterophily datasets, as the fine-grained pairwise implementation may better capture the oscillatory decay patterns characteristic of heterophily graphs.
\textbf{(3)} AFGRL~\cite{Lee2022AugmentationFreeSL} that assume local homophily (treating immediate neighbors as similar) show significantly worse performance, confirming the importance of moving beyond local connectivity assumptions to capture universal structural patterns.

\noindent\textbf{Extensibility Analysis.}
We demonstrate \mymodel{}'s extensibility across different \textit{\textbf{scales}} and \textit{\textbf{tasks}}:
\textbf{(1)} On the large-scale ogbn-arxiv graph dataset (169K nodes, 1.2M edges), following~\cite{hu2020open,thakoor2021large}, both implementations outperform existing unsupervised methods (Table~\ref{tab:ogb}). 
\textbf{(2)} For additional downstream tasks (node clustering and similarity search), our method consistently achieves superior performance across representative datasets (Table~\ref{tab:other}), with substantial improvements over methods that only consider local patterns~\cite{Lee2022AugmentationFreeSL}. 
These results validate that capturing broader structural relationships through relative similarity modeling produces versatile representations that generalize well across scales and tasks.

\subsection{Structural Pattern Analysis}

\noindent\textbf{Impact of Neighborhood Range.}
We investigate how the choice of neighborhood range $k$ affects model performance across different graph types and scales. Figure~\ref{fig:different_hops_main} shows the performance variations with different $k$ values on 3 categories of datasets: small homophily graphs (Cora), large homophily graphs (WikiCS), and heterophily graphs (Squirrel, Actor). The analysis reveals distinct patterns:
\textbf{(1)} For small homophily graph, optimal performance is achieved with $k=1$ or $2$, with accuracy declining as $k$ increases. This aligns with our empirical analysis in Section~\ref{sec:empirical}, where label consistency drops substantially (to around $20\%$ at $k=4$) in homophily graphs---when graph size is small, the number of semantically similar nodes at larger distances becomes too sparse to provide meaningful signals.
\textbf{(2)} For large homophily graph and heterophily graphs, performance generally improves with increasing $k$ until $k=3$ or $4$. This suggests that larger graphs benefit from capturing broader structural relationships, while heterophily graphs require a wider range to effectively model their oscillatory decay patterns characterized in Section~\ref{sec:theory}.
\textbf{(3)} Performance stabilizes or declines beyond $k=4$ across most datasets, supporting our default setting. 
Please refer to Appendix~\ref{appendix:experiments:neighbor-range} for more results.

\noindent\textbf{Preservation of Label Consistency Pattern.}
To verify whether our learned representations preserve the universal decay property, we analyze the similarity distributions across different structural distances in the embedding space. Figure~\ref{fig:similarity_visualization} shows the cosine similarity distributions between node representations at different hop distances on WikiCS datasets. The results reveal that: 
\textbf{(1)} The similarities exhibit an overall diminishing trend as structural distance increases, consistent with Theorem~\ref{thm:convergence}'s prediction of exponential convergence in label distributions. This empirically validates that our framework effectively captures the universal decay pattern without explicit constraints.
\textbf{(2)} While maintaining this collective trend, individual variations exist where some further neighbors show higher similarity than closer ones. This aligns with both our empirical observations in Section~\ref{sec:empirical} and the theoretical prediction that convergence occurs at a statistical rather than individual level.

\begin{table}[!t]
\centering
\scriptsize
\setlength{\tabcolsep}{0.35mm}{
\begin{tabular}{@{}lcccccc@{}}
\toprule
\textbf{Dataset} & \multicolumn{2}{c}{\textbf{WikiCS}} & \multicolumn{2}{c}{\textbf{Computers}} & \multicolumn{2}{c}{\textbf{Photo}} \\ 
\textbf{Metric}  & \textbf{NMI}    & \textbf{Sim@5}    & \textbf{NMI}      & \textbf{Sim@5}     & \textbf{NMI}    & \textbf{Sim@5}   \\ \midrule
GRACE                 &       0.4282          &    0.7754               &      0.4793             &          0.8738          &      0.6513           &      0.9155            \\
GCA                 &        0.3373         &    0.7786               &          0.5278         &           0.8826        &       0.6443          &     0.9112             \\
BGRL                 &       0.3969          &     0.7739              &       0.5364            &         0.8947           &       0.6841          &     0.9245             \\
AFGRL                 &       0.4132          &           0.7811        &         0.5520          &     0.8966               &      0.6563           &      0.9236            \\ \midrule
\rowcolor{gray!20} \mymodel{}\textsubscript{\textsc{Pair}}                 & \underline{0.4354}                 &      \underline{0.7919}             &         \textbf{0.5643}          &         \textbf{0.8992}           &      \textbf{0.6935}          &        \textbf{0.9302}         \\
\rowcolor{gray!20} \mymodel{}\textsubscript{\textsc{List}}                & \textbf{0.4376}                 &   \textbf{0.7951}                &        \underline{0.5627}           &       \underline{0.8988}             &       \underline{0.6917}          &         \underline{0.9297}         \\ \bottomrule
\end{tabular}}
\vspace{-0.25cm}
\caption{Performance of node clustering in terms of NMI and performance of node similarity search in terms of Sim@5.}
\label{tab:other}
\end{table}

\begin{table}[!t]
\centering
\scriptsize
\setlength{\tabcolsep}{0.3mm}{
\begin{tabular}{@{}lcccc@{}}
\toprule
 \textbf{Method} & \textbf{Cora} & \textbf{Citeseer} & \textbf{WikiCS} & \textbf{Photo} \\ \midrule
\rowcolor{gray!20} \mymodel{}\textsubscript{\textsc{Pair}} & $84.4$  & $73.7$  &  $80.14$ & $93.42$ \\
\quad\quad\quad \,-$\mathcal{L}^{\text{out}}$ &  $83.4({\color{red}\downarrow}1.0)$ & $73.1({\color{red}\downarrow}0.6)$ & $78.71({\color{red}\downarrow}1.43)$   & $93.16({\color{red}\downarrow}0.26)$  \\
\rowcolor{gray!20} \mymodel{}\textsubscript{\textsc{List}} & $84.5$  & $73.6$  &  $80.16$ & $93.40$ \\
\quad\quad\quad \,-$\mathcal{L}^{\text{out}}$ & $83.4({\color{red}\downarrow}1.1)$  & $73.3({\color{red}\downarrow}0.3)$ & $78.94({\color{red}\downarrow}1.22)$   & $93.18({\color{red}\downarrow}0.22)$ \\
\bottomrule
\end{tabular}}
\vspace{-0.25cm}
\caption{Performance of \mymodel{} adopting different design.}
\vspace{-0.55cm}
\label{table:infonce_var}
\end{table}

\subsection{Framework Analysis}

\noindent\textbf{Collective Similarity Design.}
We compare our collective similarity formulation $\mathcal{L}^{\text{in}}$ (Equation~\ref{equation:in}) against individual similarity summation $\mathcal{L}^{\text{out}}$ (Equation~\ref{equation:out}) on 4 representative datasets. Table~\ref{table:infonce_var} shows consistent performance degradation ($>$1\% drop on Cora and WikiCS) when switching to individual similarity summation. This validates the design motivation in Section~\ref{sec:learning_framework}---while nodes at different distances may share labels with the anchor node (as established in our theoretical analysis), enforcing individual similarity constraints fails to capture the statistical nature of these relationships. Our collective formulation, by comparing aggregate similarities between groups, maintains robustness to individual variations while preserving the broader structural patterns.

\noindent\textbf{Relative Similarity Control.}
We analyze the impact of threshold $\alpha$ by varying it in $[0.0001, 1.0]$ while keeping other parameters fixed. Figure~\ref{fig:different_margins} shows that: 
\textbf{(1)} Both implementations achieve optimal performance at moderate $\alpha$ values---extremely small thresholds overly restrict similarity relationships, while large values fail to distinguish between different structural distances.
\textbf{(2)} \mymodel{}\textsubscript{\textsc{List}} consistently requires smaller optimal $\alpha$ values than \mymodel{}\textsubscript{\textsc{Pair}}, reflecting their different granularities in modeling structural relationships. This empirically supports our framework's flexibility in adapting to different similarity modeling strategies while maintaining effective relative similarity control.

\begin{figure}[!t]
\centering
\vspace{-0.35cm}
\includegraphics[width=0.85\linewidth]{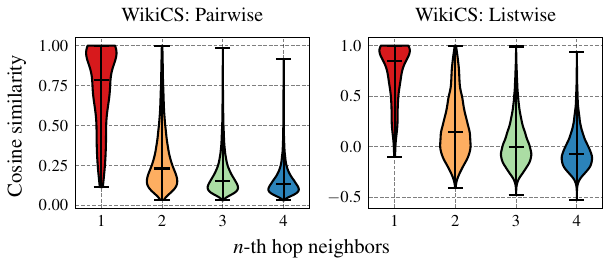}
\vspace{-0.35cm}
\caption{Similarities between the anchor nodes' and associated n-th hop neighbors' representations.}
\label{fig:similarity_visualization}
\vspace{-0.05cm}
\end{figure}

\begin{figure}[!t]
\centering
\vspace{-0.35cm}
\includegraphics[width=0.85\linewidth]{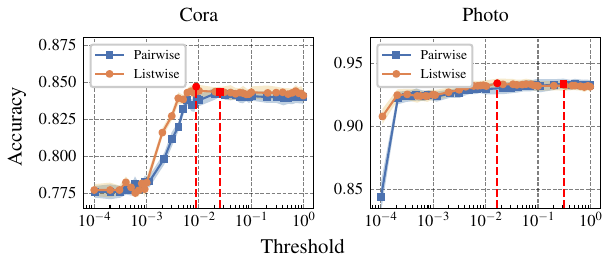}
\vspace{-0.35cm}
\caption{Results of \mymodel{} with different threshold of $\alpha$.}
\vspace{-0.55cm}
\label{fig:different_margins}
\end{figure}

\section{Related Work}

\noindent\textbf{Graph Contrastive Learning.}
GCL methods have followed the computer vision paradigm of maximizing agreement between different views. Early approaches rely on data augmentations~\cite{Velickovic2019DeepGI,Zhu2020DeepGC}, while recent works explore augmentation-free alternatives~\cite{Zhang2022COSTACF,Mo2022SimpleUG,Lee2022AugmentationFreeSL,ning2022graph} or incorporate graph-specific properties~\cite{li2023homogcl,xiao2024simple}. Latest advances propose various similarity modeling strategies, from multi-scale contrastive learning~\cite{zheng2022toward} to prototypical approaches~\cite{lin2022prototypical} and dual contrastive mechanisms~\cite{peng2023dual}. However, all these methods fundamentally operate under the ``absolute similarity" paradigm, facing inherent limitations in graphs due to semantic validity and verification challenges. In contrast, our work leverages broader structural patterns through relative similarity naturally encoded in graphs.

\noindent\textbf{Similarity Patterns in Graphs.}
Traditional graph similarity analysis primarily focuses on characterizing local neighborhood relationships through homophily-heterophily dichotomy~\cite{platonov2024characterizing}. Recent works attempt to explore broader similarity patterns: some leverage path-aware measures~\cite{rossetti2021conformity} or similarity-based path sampling~\cite{zhou2024pathmlp}, while others analyze nonlinear similarity decay in communities~\cite{martinez2022assessing}. However, these methods either focus on engineering better similarity metrics or analyzing specific graph properties, without providing a fundamental understanding of how node relationships systematically vary with structural distance. In contrast, our work reveals and theoretically proves a universal pattern of similarity decay across both homophily and heterophily graphs, providing a principled foundation for understanding graph-inherent similarity patterns.

\section{Conclusion}

This work fundamentally reimagines similarity in graph contrastive learning by discovering a universal pattern: the systematic diminishing of label consistency with structural distance. Through rigorous analysis, we establish random walk guarantees for this pattern in both homophily and heterophily graphs. Building on these insights, we develop \mymodel{}, a principled framework that preserves these inherent patterns through collective similarity objectives. Extensive experiments validate that leveraging natural relative similarity yields superior performance across diverse graph types, opening new directions for robust graph representation learning.

\section{Acknowledgements}
This work is supported by 
the Science and Technology Development Fund (FDCT), Macau SAR (file no. 0123/2023/RIA2, 001/2024/SKL),
the National Natural Science Foundation of China (Grant Nos. 92470204, 62406306 and 62406056),
and the State Key Laboratory of Internet of Things for Smart City (University of Macau) No. SKL-IoTSC(UM)-2024-2026/ORP/GA02/2023.


\appendix

\section{Detailed Theoretical Analysis}
\label{appendix:theory}

We provide a rigorous proof of Theorem~\ref{thm:convergence} and Corollary~\ref{cor:lc_decay} of the main text, building upon classical results from Markov chain theory~\cite{levin2017markov}.

\subsection{Notation and Preliminaries}
We consider an undirected graph $G=(V,E)$, where $V$ is the set of nodes and $E$ is the set of edges. Let $l(u)$ denote the label of node $u$, where labels are in $\{1,\ldots,c\}$. For the adjacency matrix $A$ of $G$, $A_{uv}=1$ if nodes $u,v$ are connected (thus $A_{uv}=A_{vu}$), and $\text{deg}(u)=\sum_v A_{uv}$ denotes node degree.
The label-level transition probability $T_{ij}$ represents the probability of transitioning from any node with label $i$ to any node with label $j$, aggregated from node-level transitions where $v_{\text{next}}$ denotes the next node in the random walk.

\subsection{Random Walk Framework}
\label{appendix:theory:random_walk}

Consider a graph with $c$ distinct labels where each node $u$ has a label $l(u) \in \{1,\ldots,c\}$.
Our analysis bridges node-level and label-level random walks, extending the concept introduced in Section~\ref{sec:theory} of the main text.
While traditional random walks operate at the node level with transition probability $P_{uv} = A_{uv}/\text{deg}(u)$, we derive our label-level transition matrix by aggregating node-level transitions between label classes. This aggregation preserves the Markov property while allowing us to analyze label dynamics directly.

The relationship between node-level and label-level random walks can be formalized as follows. For any node $u$ with label $i$, the probability of transitioning to any node with label $j$ is:

\begin{equation}
\label{eq:node_to_label}
P(l(v_{\text{next}}) = j | u) = \frac{\sum_{v:l(v)=j}A_{uv}}{\text{deg}(u)}
\end{equation}

The label-level transition probability $T_{ij}$ is then obtained by averaging over all nodes with label $i$, weighted by their degrees:

\begin{equation}
\label{eq:level_connection}
T_{ij} = \frac{\sum_{u:l(u)=i}\text{deg}(u)P(l(v_{\text{next}}) = j | u)}{\sum_{u:l(u)=i}\text{deg}(u)}
\end{equation}

The degree-weighted averaging in Equation~\eqref{eq:level_connection} is crucial for preserving the random walk dynamics. Consider two nodes $u_1, u_2$ with the same label $i$ but different degrees $\text{deg}(u_1) \gg \text{deg}(u_2)$. A random walker is more likely to visit $u_1$ than $u_2$ proportional to their degrees. Therefore, when aggregating to the label level, $u_1$'s transition patterns should have proportionally more influence on $T_{ij}$. Mathematically:

\begin{equation}
\label{eq:degree_weight_explanation}
\begin{aligned}
T_{ij} &= \frac{\sum_{u:l(u)=i}\text{deg}(u)P(l(v_{\text{next}}) = j | u)}{\sum_{u:l(u)=i}\text{deg}(u)} \\
&= \sum_{u:l(u)=i} \frac{\text{deg}(u)}{\sum_{v:l(v)=i}\text{deg}(v)} P(l(v_{\text{next}}) = j | u)
\end{aligned}
\end{equation}

The second form explicitly shows that each node's contribution is weighted by its relative degree within its label class. This weighting ensures that the label-level transition probabilities accurately reflect the underlying node-level random walk dynamics, particularly in heterogeneous networks where node degrees vary significantly.

This aggregation preserves the essential properties needed for our theoretical analysis.

\subsection{Basic Properties}
\label{appendix:theory:basic_properties}

We establish key properties of the transition matrix $T$:

\begin{lemma}[Basic Properties]
\label{lemma:transition}
For a connected graph where each node has a self-loop (this assumption simplifies ensuring $T_{ii}>0$, which directly leads to aperiodicity for the label-level Markov chain $T$), the label transition matrix $T$ satisfies three key properties:
\begin{equation}
\begin{aligned}
&\text{(1) Row stochasticity: } \sum_j T_{ij} = 1 \text{ and } T_{ij} \geq 0 \text{ for all } i,j \\
&\text{(2) Irreducibility: } \exists k \text{ such that } (T^k)_{ij} > 0 \text{ for all } i,j \\
&\text{(3) Aperiodicity: } \gcd\{k: (T^k)_{ii} > 0\} = 1 \text{ for all } i
\end{aligned}
\end{equation}
\end{lemma}

\begin{proof}
Row stochasticity can be proven by direct calculation:
\begin{align*}
\sum_j T_{ij} &= \sum_j \frac{\sum_{u:l(u)=i}\text{deg}(u)P(l(v_{\text{next}}) = j | u)}{\sum_{u:l(u)=i}\text{deg}(u)} \\
&= \frac{\sum_{u:l(u)=i}\text{deg}(u)\sum_j P(l(v_{\text{next}}) = j | u)}{\sum_{u:l(u)=i}\text{deg}(u)} \\
&= \frac{\sum_{u:l(u)=i}\text{deg}(u) \cdot 1}{\sum_{u:l(u)=i}\text{deg}(u)} = 1
\end{align*}

For irreducibility, we show that for any labels $i,j$, there exists a sequence of labels $l_1,\ldots,l_k$ such that:
\begin{equation}
T_{il_1}T_{l_1l_2}\cdots T_{l_kj} > 0
\end{equation}
This follows from graph connectivity: for any nodes $u,v$ with labels $i,j$ respectively, there exists a path $u=v_0,v_1,\ldots,v_m=v$ where $A_{v_kv_{k+1}}>0$. Let $l_k=l(v_k)$. Then:
\begin{equation}
\begin{aligned}
T_{l_k l_{k+1}} &= \frac{\sum_{u:l(u)=l_k}\text{deg}(u)\,P\bigl(l(v_{\text{next}})=l_{k+1}\mid u\bigr)}{\sum_{u:l(u)=l_k}\text{deg}(u)} \\[6pt]
&= \frac{\sum_{u:l(u)=l_k}\text{deg}(u)\,\frac{\sum_{w:l(w)=l_{k+1}}A_{u w}}{\text{deg}(u)}}{\sum_{u:l(u)=l_k}\text{deg}(u)} \\[6pt]
&= \frac{\sum_{u:l(u)=l_k}\sum_{w:l(w)=l_{k+1}}A_{u w}}{\sum_{u:l(u)=l_k}\text{deg}(u)} \\[6pt]
&\geq \frac{A_{v_k v_{k+1}}}{\sum_{u:l(u)=l_k}\text{deg}(u)} \\[6pt]
&> 0
\end{aligned}
\end{equation}

For aperiodicity, self-loops ensure $(T^1)_{ii}>0$. Moreover, for any $k\geq 1$, we can construct a path of length $k$ from $i$ to itself by inserting self-loops, proving $\gcd\{k: (T^k)_{ii} > 0\} = 1$.
\end{proof}

\subsection{Proof of Label Distribution Convergence}
\label{appendix:theory:proof_convergence}

The preservation of Markov chain properties through our label-level aggregation allows us to apply classical results:

\begin{proof}[Proof of Theorem~\ref{thm:convergence}]
By Lemma~\ref{lemma:transition}, $T$ defines an irreducible and aperiodic Markov chain. The Perron-Frobenius theorem~\cite{levin2017markov} ensures that:

A unique stationary distribution $\pi$ exists satisfying $\pi T = \pi$. The eigenvalues of $T$ satisfy $1 = \lambda_1 > |\lambda_2| \geq \cdots \geq |\lambda_c|$. For any initial distribution $\mu_0$, $\lim_{k \to \infty} \mu_0 T^k = \pi$.

For convergence rate, we use spectral decomposition. Let $v_1,\ldots,v_c$ be the right eigenvectors of $T$. Any initial distribution $\mu_0$ can be written as:
\begin{equation}
\mu_0 = \pi + \sum_{j=2}^c c_j v_j
\end{equation}
where $c_j$ are coefficients. After $k$ steps:
\begin{equation}
\mu_0 T^k = \pi + \sum_{j=2}^c c_j \lambda_j^k v_j
\end{equation}

By the Convergence Rate Theorem for finite irreducible aperiodic Markov chains~\cite{levin2017markov}, we have the convergence bound:
\begin{equation}
\|\mu_0 T^k - \pi\|_1 \leq C|\lambda_2|^k
\end{equation}
where $C$ is a constant depending on the initial distribution $\mu_0$ and the stationary distribution $\pi$.

For the stationary distribution, we verify directly:
\begin{equation}
\begin{aligned}
(\pi T)_j &= \sum_i \pi_i T_{ij} \\
&= \sum_i \frac{\sum_{u:l(u)=i}\text{deg}(u)}{\sum_w \text{deg}(w)} \cdot \\ 
& \quad\quad \frac{\sum_{u:l(u)=i}\text{deg}(u)P(l(v_{\text{next}}) = j | u)}{\sum_{u:l(u)=i}\text{deg}(u)} \\
&= \sum_i \frac{\sum_{u:l(u)=i}\text{deg}(u)}{\sum_w \text{deg}(w)} \cdot \frac{\sum_{u:l(u)=i}\text{deg}(u)\frac{\sum_{v:l(v)=j}A_{uv}}{\text{deg}(u)}}{\sum_{u:l(u)=i}\text{deg}(u)} \\
&= \sum_i \frac{\sum_{u:l(u)=i}\text{deg}(u)}{\sum_w \text{deg}(w)} \cdot \frac{\sum_{u:l(u)=i}\sum_{v:l(v)=j}A_{uv}}{\sum_{u:l(u)=i}\text{deg}(u)} \\
&= \frac{1}{\sum_w \text{deg}(w)} \sum_i \sum_{u:l(u)=i}\sum_{v:l(v)=j}A_{uv} \\
&= \frac{1}{\sum_w \text{deg}(w)} \sum_u\sum_{v:l(v)=j}A_{uv} \\
&= \frac{1}{\sum_w \text{deg}(w)} \sum_{v:l(v)=j}\sum_u A_{uv} \\
&= \frac{1}{\sum_w \text{deg}(w)} \sum_{v:l(v)=j}\text{deg}(v) \quad \text{(\small since $\sum_u A_{uv} = \text{deg}(v)$)} \\
&= \frac{\sum_{v:l(v)=j}\text{deg}(v)}{\sum_w \text{deg}(w)} = \pi_j
\end{aligned}
\end{equation}

Notably, this stationary distribution $\pi_j = \frac{\sum_{u:l(u)=j}\text{deg}(u)}{\sum_v \text{deg}(v)}$ has a clear physical interpretation: it represents the proportion of edges connected to nodes with label $j$ in the graph. This explains why $\pi_j$ serves as the natural limit for label transition probabilities, as it reflects the structural prominence of label $j$ in terms of edge connections.
\end{proof}

\subsection{Proof of Label Consistency Decay}
\label{appendix:theory:proof_decay}

\begin{proof}[Proof of Corollary~\ref{cor:lc_decay}]
The result follows from Theorem~\ref{thm:convergence}. For $k=0$, $\mathit{LC}_{\text{prob}}(0)=1$ by definition since we start at a node with label $i$. As $k \to \infty$, $\mathit{LC}_{\text{prob}}(k) = p_k(i|i) \to \pi_i$ by the theorem.

To prove $\pi_i < 1$, note that in a connected graph, there must exist nodes with labels other than $i$ and their degrees sum to a positive value (due to connectivity). Therefore:
\begin{equation}
\small
\sum_v \text{deg}(v) =\sum_{u:l(u)=i}\text{deg}(u) + \sum_{u:l(u)\neq i}\text{deg}(u) > \sum_{u:l(u)=i}\text{deg}(u)
\end{equation}
This implies:
\begin{equation}
\pi_i = \frac{\sum_{u:l(u)=i}\text{deg}(u)}{\sum_v \text{deg}(v)} < 1
\end{equation}

The exponential decay follows from the spectral decomposition:
\begin{equation}
\mathit{LC}_{\text{prob}}(k) = (T^k)_{ii} = \pi_i + \sum_{j=2}^c c_j \lambda_j^k v_{ji}
\end{equation}
where $|\lambda_j| < 1$ for $j\geq 2$.
\end{proof}

\subsection{Contrasting Decay Mechanisms in Homophily and Heterophily Graphs}
\label{appendix:theory:contrasting}

Our theoretical analysis in Appendix~\ref{appendix:theory:proof_convergence} and~\ref{appendix:theory:proof_decay} proves the asymptotic decay of label consistency. Here we provide an intuitive interpretation of the observed decay patterns in real-world graphs through a simplified lens, though a complete theoretical characterization remains an open challenge.

\subsubsection{Idealized assumptions for two-label analysis} 
The following analysis of two-label graphs relies on two key assumptions to isolate the core mechanism:
\begin{itemize}
    \item Label symmetry: Both labels have equal size (number of nodes) and degree distribution
    \item Uniform transition: All nodes within a label class share identical probabilities of connecting to other labels
\end{itemize}
These assumptions simplify the transition matrix to a symmetric form, enabling closed-form eigenvalue solutions. Real-world graphs may deviate from this idealized scenario due to label imbalance or heterogeneous connectivity.

\subsubsection{Homophily graph analysis} 
For homophily graphs (8 instances visualized in Figure~\ref{fig:motivation} left), $\mathit{LC}_{\text{emp}}(k)$ shows smooth decay patterns. To understand this behavior, consider the structure of transition matrix $T$ in homophily graphs. From Equation~\eqref{eq:level_connection}:
\begin{equation}
\label{eq:homophily_tii}
\begin{aligned}
T_{ii} &= \frac{\sum_{u:l(u)=i}\text{deg}(u)P(l(v_{\text{next}}) = i | u)}{\sum_{u:l(u)=i}\text{deg}(u)} \\
&= \frac{\sum_{u:l(u)=i}\text{deg}(u)\frac{\sum_{v:l(v)=i}A_{uv}}{\text{deg}(u)}}{\sum_{u:l(u)=i}\text{deg}(u)} \\
&= \frac{\sum_{u:l(u)=i}\sum_{v:l(v)=i}A_{uv}}{\sum_{u:l(u)=i}\text{deg}(u)} \gg T_{ij} \quad \text{for } j \neq i
\end{aligned}
\end{equation}

\subsubsection{Transition matrix simplification} 
Under the idealized assumptions, the dominant self-transition probabilities lead to transition matrices of the form when examining the simplest two-label case:
\begin{equation}
T \approx \begin{bmatrix} 
p & 1-p \\
1-p & p
\end{bmatrix} \quad \text{where} \quad p \gg 0.5
\end{equation}

By solving the characteristic equation $\det(T - \lambda I) = 0$, such matrices have eigenvalues $\lambda_1 = 1$ and $\lambda_2 = 2p-1 > 0$. When applied to the spectral decomposition from Theorem~\ref{thm:convergence}:
\begin{equation}
\label{eq:empirical_spectral}
\mathit{LC}_{\text{prob}}(k) = (T^k)_{ii} = \pi_i + \sum_{j=2}^c c_j \lambda_j^k v_{ji}
\end{equation}
the positive second eigenvalue leads to smooth exponential decay, as $(\lambda_2)^k$ remains positive throughout the decay process.

\subsubsection{Heterophily graph analysis} 
In contrast, heterophily graphs (3 instances visualized in Figure~\ref{fig:motivation} right) exhibit markedly different behavior: sharp initial decay followed by oscillatory patterns while maintaining an overall decreasing trend. This behavior stems from dominant cross-label transitions:
\begin{equation}
\label{eq:heterophily_tij}
T_{ij} = \frac{\sum_{u:l(u)=i}\sum_{v:l(v)=j}A_{uv}}{\sum_{u:l(u)=i}\text{deg}(u)} \gg T_{ii} \quad \text{for } j \neq i
\end{equation}

Under the same two-label simplification, this creates transition matrices approximating:
\begin{equation}
T \approx \begin{bmatrix}
1-p & p \\
p & 1-p
\end{bmatrix} \quad \text{where} \quad p \gg 0.5
\end{equation}

These matrices have eigenvalues $\lambda_1 = 1$ and $\lambda_2 = 1-2p < 0$. When substituted into the spectral decomposition, the negative second eigenvalue causes oscillation in the decay process, as $(\lambda_2)^k$ alternates between positive and negative values with increasing $k$, while still maintaining the overall decay due to $|\lambda_2| < 1$.

\subsubsection{Limitations and extensions} 
While this analysis explains the oscillatory vs. smooth decay dichotomy under idealized conditions, extending it to graphs with label imbalance (e.g., 90\%-10\% label distribution) or heterogeneous connectivity (e.g., high-degree hubs) requires more complex spectral analysis beyond the scope of this work. The core eigenvalue mechanism ($\lambda_2$ sign determining decay pattern) persists in general cases, but the simplified closed-form solutions only hold under the stated assumptions.

\subsubsection{Open challenges and summary} 
While this simplified analysis provides intuitive insights into the observed decay patterns, a complete theoretical characterization accounting for multiple labels, arbitrary graph structures, and varying degrees of homophily/heterophily remains an open challenge. The key finding is that despite their different local behaviors, both graph types consistently exhibit decay in label consistency, as rigorously proven in Appendix~\ref{appendix:theory:proof_convergence} and~\ref{appendix:theory:proof_decay}.
\section{Detailed Method Analysis}
\label{appendix:method_analysis}

\subsection{Comparison with Previous Methods}

Table~\ref{tab:method_comparison} presents a systematic comparison between our framework and representative GCL methods. We analyze this comparison from three key aspects: data augmentation, view generation, and similarity modeling.

Early GCL methods like GRACE~\cite{Zhu2020DeepGC} and GCA~\cite{Zhu2021GraphCL} follow the traditional contrastive learning~\cite{qiao2022rpt} paradigm from computer vision, employing data augmentation to create multiple views of the same graph and enforcing absolute similarity between corresponding nodes across views. While these methods have shown promising results, they face fundamental challenges when applied to graphs.
The first major challenge lies in view generation validity. Unlike images where operations such as rotation preserve semantic meaning, graph augmentations can significantly alter the underlying structure and properties of the graph. This issue becomes particularly pronounced when dealing with graphs representing molecular structures or social relationships, where each edge carries crucial semantic information.
Another significant challenge concerns similarity specification in graphs. Determining which node pairs should be considered similar proves far more complex than in visual data. While humans can readily verify if two image views represent the same object, judging similarity between graph views often exceeds human intuition.

Later approaches like SUGRL~\cite{Mo2022SimpleUG} and AFGRL~\cite{Lee2022AugmentationFreeSL} attempted to address the view generation challenge by eliminating the need for data augmentation. However, they still maintain the core assumption of specifying absolute similarity between certain node pairs, inheriting the similarity specification challenge.

Our framework takes a fundamentally different approach by leveraging the natural similarity patterns encoded in graph structure~\cite{qiao2024dual,dong2023adaptive,dong2024temporal}. Instead of artificially constructing similar pairs, we preserve the relative similarity relationships that emerge from the graph's inherent structural-semantic properties. This paradigm shift offers several key advantages.
By working with the graph's inherent structure rather than artificial views, our approach naturally avoids the challenge of generating valid augmentations. The relative similarity framework acknowledges the noisy nature of structural relationships, enabling more robust representation learning~\cite{ning2024fedgcs}. Furthermore, this approach naturally handles both homophily and heterophily graphs by focusing on relative similarity patterns that transcend local connectivity preferences. Most importantly, our method is grounded in theoretical understanding of how label consistency decays with structural distance, providing a principled basis for representation learning.

\begin{table}[!t]
\scriptsize
\setlength\tabcolsep{0.75pt}
\centering
\begin{tabular}{lccc}
\toprule
Method & Data Augmentation & View Generation & Absolute Similarity \\
\midrule
GRACE~\cite{Zhu2020DeepGC} & \checkmark & \checkmark & \checkmark \\
GCA~\cite{Zhu2021GraphCL} & \checkmark & \checkmark & \checkmark \\
BGRL~\cite{thakoor2021large} & \checkmark & \checkmark & \checkmark \\
SUGRL~\cite{Mo2022SimpleUG} & $\times$ & \checkmark & \checkmark \\
AFGRL~\cite{Lee2022AugmentationFreeSL} & $\times$ & \checkmark & \checkmark \\
\midrule
Ours & $\times$ & $\times$ & $\times$ \\
\bottomrule
\end{tabular}
\caption{Comparison of key components between existing GCL methods and our approach.}
\label{tab:method_comparison}
\end{table}

\subsection{Detailed Complexity Analysis}

For a graph $\mathcal{G}$ with $N$ nodes and average degree $d$, we analyze the computational complexity in terms of similarity computations $\theta(\cdot,\cdot)$, as this is the dominant operation in contrastive learning frameworks.
The average $k$-th neighborhood size $|\mathcal{N}(v_i)^{[k]}|=k^d$. 
Next, we introduce how to calculate the complexity of \mymodel{}\textsubscript{\textsc{Pair}} and \mymodel{}\textsubscript{\textsc{List}}, respectively.

\subsubsection{Pairwise Relative Similarity}
~\label{appendix: pair-wise}
In Equation~\ref{equation:pair-wise}, for each anchor node, we will calculate the corresponding ``outer loop'', and each ``outer loop'' includes several ``inner loop'' as:

\begin{flalign}
    \vspace{-4mm}
    \scalemath{0.65}{
        \mathcal{L} = - \sum\limits_{\mathbf{v}_i\in \mathcal{V}} \frac{1}{k} \sum\limits_{j=1}^{k} \underbrace{ \sum\limits_{m=1}^{k-j+1} \log \underbrace{\left[ \min\{\frac{\sum\limits_{\mathbf{h}_\ast \in \mathbb{H}^{[j]}_i}\exp(\frac{s(\mathbf{h}_i, \mathbf{h}_\ast)}{\tau})}{\sum\limits_{\mathbf{h}_\ast \in \mathbb{H}^{[j]}_i}\exp(\frac{s(\mathbf{h}_i, \mathbf{h}_\ast)}{\tau}) + \sum\limits_{\mathbf{h}_\diamond \in \mathbb{H}^{[j+m]}_i}\exp(\frac{s(\mathbf{h}_i, \mathbf{h}_\diamond)}{\tau})}, \alpha \}\right]}_\text{inner loop}}_\text{outer loop}.}
    \label{equation:pair-wise ranking infoNCE analysis}
    \vspace{-4mm}
    \end{flalign}
    
For the $j$-th neighborhood node of the anchor node $v_i$, it includes $k-j+1=k$ rounds of ``inner loop'' for relative similarity modeling. 
For each ``inner loop'', it includes $d^{j}$ times of $s(\mathbf{h}_i, \mathbf{h}_\ast)$, and the total number for one ``outer loop'' is $d^j+d^{j+1}+\cdots+d^{k+1}$. 
Then, the number of computation of $s(\cdot)$ can be listed as 

\begin{equation}
    \small
    \vspace{-2mm}
    \begin{aligned}
        \overbrace{\qquad\qquad\qquad\qquad\qquad \qquad\qquad\qquad\qquad\qquad\qquad}^{ inner \enspace loop} \\
    \makecell[c]{outer \\loop}\left\{
    \begin{aligned}
        & 1-th\enspace hop \quad d + d^2 \quad d+d^3 \cdots \quad d + d^k \quad d + d^{k+1} \\
        & 2-th\enspace hop \quad d^2+d^3 \quad \cdots \quad d^2 + d^k \quad d^2 + d^{k+1} \\
        & \qquad\vdots \qquad\qquad\qquad \begin{turn}{90}
            $\ddots$
        \end{turn} \\
        & k-th\enspace hop \quad d^k + d^{k+1} \\
    \end{aligned}
    \right.
    \end{aligned}
\label{equation:appendix-complex-pairwise}
\end{equation}

\vspace{2mm}
Then, for the anchor node $v_i$, we need 
\begin{align}
\vspace{-2mm}
\small
\begin{split}
     &k\cdot d + k\cdot d^2 + \cdots + k\cdot d^k + k\cdot d^{k+1} \\ 
   =& k\cdot(d + d^2 + \cdots + d^k + d^{k+1}) \\
   =& k\sum \limits_{i=1}^{k+1} d^i  .
\end{split}
\end{align}

Since $\mathcal{G}$ contains $N$ nodes, the total computation times of $s(\cdot)$ is $N \cdot k\sum \limits_{i=1}^{k+1} d^i$.
    

\subsubsection{Listwise Relative Similarity}
~\label{appendix: list-wise}
In Equation~\ref{equation:list-wise}, for each anchor node, we only need to calculate the ``outer loop'', because no ``inner loop'' is in \mymodel{}\textsubscript{\textsc{List}}.
\begin{equation}
   \small
    \mathcal{L} = - \sum\limits_{\mathbf{v}_i\in \mathcal{V}} \frac{1}{k} \sum\limits_{j=1}^{k} \log \underbrace{\left[ \min \{ \frac{\sum\limits_{\mathbf{h}_\ast \in \mathbb{H}^{[j]}_i}\exp(\frac{s(\mathbf{h}_i, \mathbf{h}_\ast)}{\tau})}{\sum\limits_{j'=j}^{k+1} \sum\limits_{\mathbf{h}_\diamond \in \mathbb{H}^{[j']}_i} \exp(\frac{s(\mathbf{h}_i, \mathbf{h}_{\diamond})}{\tau}) },  \alpha \}\right]}_{outer\enspace loop},
    \label{equation:list-wise ranking infoNCE analysis}
\end{equation} 
Therefore, for $j$-th hop neighborhood of the anchor node $v_i$, the ``outer loop'' includes $d^j+d^{j+1}+\cdots+d^{k}+d^{k+1}$ times of $s(\cdot)$. 
Then, the number of computation of $s(\cdot)$ can be listed as 

\begin{equation}
    \small
    \makecell[c]{outer \\loop}\left\{
    \begin{aligned}
        & 1-th\enspace hop \quad d + d^2 + d^3 + \cdots + d^k + d^{k+1}\\
        & 2-th\enspace hop \quad d^2 + d^3 + \cdots + d^k + d^{k+1}\\
        & \qquad\vdots \qquad\qquad\qquad \begin{turn}{90}
            $\ddots$
        \end{turn} \\
        & k-th\enspace hop \quad d^k + d^{k+1}\\
    \end{aligned}
    \right.
    \label{equation:appendix-complex-listwise}
\end{equation} 


Then, for the anchor node $v_i$. we need 

\begin{align}
\small
\begin{split}
    &d+2\cdot d^2 + 3\cdot d^3 + \cdots + k\cdot d^k +d^{k+1} \\
    =& \sum \limits_{i=1}^{k} i\cdot d^i + k\cdot d^{k+1}.
\end{split}   
\end{align}
Since $\mathcal{G}$ contains $N$ nodes, the total computation times of $s(\cdot)$ is $N \cdot (\sum \limits_{i=1}^{k} i\cdot d^i + k\cdot d^{k+1})$.

\subsubsection{Summary}

Traditional GCL methods typically compute similarity between one anchor node and $Q$ samples, resulting in $\mathcal{O}(NQ)$ computations per epoch.
For our relative similarity framework with $k$-hop neighborhoods, the pairwise implementation requires $N\cdot k\sum_{i=1}^{k+1}d^i$ computations, while the listwise implementation needs $N(\sum_{i=1}^{k}i\cdot d^i+k\cdot d^{k+1})$ computations. By caching similarity scores to avoid redundant computations, both achieve an actual complexity of $\mathcal{O}(Nd(1+d+d^2))$. For large-scale graphs, we can sample a fixed number $d'$ ($d' \ll d$) of nodes from each hop, reducing the complexity to $\mathcal{O}(Nd'(1+d'+d'^2))$. This makes our framework computationally comparable to traditional GCL methods while maintaining its ability to capture natural similarity relationships.
\section{Detailed Experimental Setting}
\label{appendix:experimental-setting}

\begin{table*}[!t]
\centering
\small
\setlength{\tabcolsep}{1.00mm}{
\begin{tabular}{@{}lcrrrrrrrrrr@{}}
\toprule
\textbf{Dataset} & \textbf{Type} &  $\mathit{HM}$  & \textbf{\#Nodes} & \textbf{\#Edges} & \textbf{\#Features} & \textbf{\#Classes} & \textbf{\makecell[c]{\#1-hop \\ {\scriptsize neighbors}}} & \textbf{\makecell[c]{\#2-hop \\ {\scriptsize neighbors}}} & \textbf{\makecell[c]{\#3-hop \\ {\scriptsize neighbors}}} & \textbf{\makecell[c]{\#4-hop \\ {\scriptsize neighbors}}} & \textbf{\makecell[c]{\#5-hop \\ {\scriptsize neighbors}}} \\  \midrule
Cora                 & Homophiliy   &  0.8252  &       $2,708$           &       $10,556$           &        $1,433$             &         7           &  $3.90$   & $31.9$   &  $91.3$   & 244.9  &  438.4   \\
Citeseer                 &  Homophiliy  &  0.7062  &     $3,327$            &       $9,228$           &         $3,703$            &        6            &  $2.7$   &  $11.4$  &  $28.4$   & 52.7   &  78.1    \\
Pubmed                &  Homophiliy  & 0.7924   &      $19,717$          &       $88,651$           &         500            &            3        &  $4.5$   & $54.6$   &  334.6    & 1596.7  &  3321.4  \\
WikiCS                 &  Homophiliy  &  0.6409  &    $11,701$           &        $216,123$          &        $300$             &        10            &  $24.8$  & $1051.4$ &  $4701.8$ & 3199.0 &  647.9  \\
Computers                 &  Homophiliy  & 0.7853   &    $13,752$             &      $245,861$            &       $767$              &        10            &  $35.8$  & $1812.0$ &  $5644.3$ & 4332.6 &  1047.8  \\
Photo                 & Homophiliy  &  0.8365  &    $7,650$           &        $119,081$          &          $745$           &          8          &  $31.1$  & $770.7$  &  $1716.5$ & 2162.3 &  1729.2 \\
Co.CS                 & Homophiliy  &  0.8320  &    $18,333$           &       $81,894$           &          $6,805$           &       15             &  $8.9$   & $98.7$   &  $765.0$  & 3225.5 &  6158.7  \\
ogbn-arxiv                 & Homophiliy  &  0.5532  &    $169,343$           &       $1,166,243$           &          $128$           &       40             &  $7.0$   & $2038.3$   &  $8234.6$  & 16160.1 &  19269.8  \\
Chameleon                 &  Heterophily  &  0.2425  &    $2,277$           &       $36,101$           &          $2,325$           &       5             &  $27.6$   & $531.1$   &  $507.7$  & 761.6 &  298.7 \\
Squirrel                 &  Heterophily  &  0.2175  &    $5,201$           &       $217,073$           &          $2,089$           &       5             &  $76.3$   & $1615.8$   &  $1846.5$  & 1171.4 &  413.1  \\
Actor                 &  Heterophily  &  0.1848  &    $7,600$           &       $33,544$           &          $931$           &       5             &  $3.9$   & $96.7$   &  $482.1$  & 1444.4 &  1651.5 \\
\bottomrule
\end{tabular}}
\caption{The characteristics and statistics of datasets, where $\mathit{HM}=\mathit{LC}_{\text{emp}}(1)$ represents the homophily metric, and ``\#$n$-hop neighbors" represents the average number of $n$-th hop neighbors.}
\label{tab:datasets_statistics}
\end{table*}

\subsection{Datasets}
\label{appendix:experimental-setting:datasets}
To comprehensively validate the universality of our approach, we conduct extensive experiments on 11 real-world benchmark datasets spanning diverse domains (citation networks, social networks, e-commerce), varying scales (from 2K to 169K nodes), and different structural properties (8 homophily graphs and 3 heterophily graphs). These datasets include Cora, Citeseer, Pubmed, WikiCS, Amazon-Computers, Amazon-Photo, Coauthor-CS, and ogbn-arxiv (homophily), along with Chameleon, Squirrel, and Actor (heterophily). The characteristics and statistics of these datasets are summarized in Table \ref{tab:datasets_statistics}. Here we provide detailed descriptions of each dataset category:

\begin{itemize}[left=0pt]
\item \textbf{Citation Networks} (Cora, CiteSeer, PubMed, and ogbn-arxiv): Nodes represent academic papers and edges denote citation relationships~\cite{sen2008collective,hu2020open}. Each node contains bag-of-words features derived from paper content, with labels indicating the paper's research area. We use their public train/valid/test splits.

\item \textbf{Co-occurrence Networks} (WikiCS, Amazon-Computers, Amazon-Photo, and Coauthor-CS): These graphs capture various forms of co-occurrence relationships. WikiCS connects computer science articles through hyperlinks, with node features being the average of GloVe word embeddings~\cite{pennington2014glove}. Amazon datasets~\cite{mcauley2015image} link co-purchased products, while Coauthor-CS derived from MAG~\cite{sinha2015overview} connects authors through co-authorship. For Amazon and Coauthor datasets without standard splits, we randomly partition nodes into train/valid/test sets (10\%/10\%/80\%). For WikiCS, we use its 20 provided public splits.

\item \textbf{Heterophily Graphs} (Chameleon, Squirrel, and Actor): These datasets naturally exhibit heterophily characteristics. Chameleon and Squirrel are Wikipedia page networks where edges represent mutual links~\cite{rozemberczki2021multi}, while Actor is derived from a film-industry network where edges indicate co-occurrence on Wikipedia pages~\cite{Pei2020GeomGCNGG}. We use the public splits provided by Geom-GCN~\cite{Pei2020GeomGCNGG}.
\end{itemize}

As shown in Table \ref{tab:datasets_statistics} and our analysis in Section~\ref{sec:empirical}, the number of semantically similar neighbors (those sharing the same label) becomes particularly small beyond 4 hops, while the total number of neighbors either decreases or grows at a diminishing rate after the 5th hop. Therefore, in our experiments, we set the neighborhood range $k$ (in Equations~\ref{equation:pair-wise} and \ref{equation:list-wise}) to 4, focusing on modeling relative similarity patterns within meaningful structural distances.

\subsection{Baselines}
\label{appendix:experimental-setting:baselines}
To rigorously evaluate the effectiveness of \mymodel{}, we conduct comprehensive comparisons against 20 representative baselines covering the full spectrum of graph representation learning and GCL approaches: 
including 2 supervised methods ({\it i.e.}, GCN~\cite{welling2016semi} and GAT~\cite{Velickovic2018GraphAN}), 
3 autoencoder based methods ({\it i.e.}, GAE, VGAE~\cite{Kipf2016VariationalGA} and GraphMAE~\cite{Hou2022GraphMAESM}),
9 augmentation-based GCL methods ({\it i.e.}, DGI~\cite{Velickovic2019DeepGI}, MVGRL~\cite{hassani2020contrastive}, GRACE~\cite{Zhu2020DeepGC}, GCA~\cite{Zhu2021GraphCL}, BGRL~\cite{thakoor2021large}, G-BT~\cite{bielak2022graph}), CCA-SSG~\cite{Zhang2021FromCC}, gCooL~\cite{li2022graph} and HomoGCL~\cite{li2023homogcl}), 
5 augmentation-free GCL methods ({\it i.e.}, COSTA~\cite{Zhang2022COSTACF}, SUGRL~\cite{Mo2022SimpleUG}, AFGRL~\cite{Lee2022AugmentationFreeSL} and SP-GCL~\cite{wang2023singlepass} and GraphACL~\cite{xiao2024simple}),
and 1 multi-task self-supervised learning based method AUTOSSL~\cite{jin2021automated}.
This diverse set of baselines ensures a thorough evaluation against both classical and state-of-the-art approaches in graph representation learning.

\begin{figure*}[!t]
\centering
\includegraphics[width=1.0\linewidth]{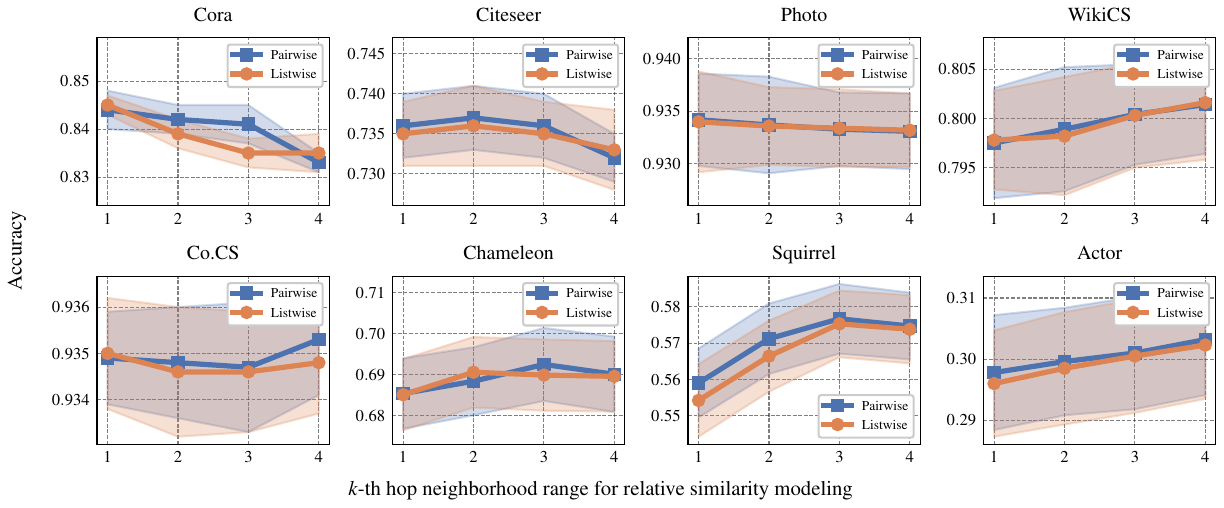}
\caption{Results of \mymodel{} with neighborhood range changing as $k=1, 2, 3, 4$, respectively.}
\label{fig:different_hops}
\end{figure*}

\subsection{Evaluation Protocol}
\label{appendix:experimental-setting:evaluation}
Following standard practice in graph self-supervised learning~\cite{Velickovic2019DeepGI,Zhu2020DeepGC,thakoor2021large}, we evaluate our method using linear evaluation. Specifically, we first train the graph encoder (GCN~\cite{welling2016semi}) on the entire graph in a self-supervised manner using our relative similarity objectives. After training, we freeze the encoder and use it to generate node embeddings. These embeddings are then used to train a logistic regression classifier for node classification, with performance measured by classification accuracy on the test set.

\subsection{Implementation Details}
\label{appendix:experimental-setting:implementation}
\mymodel{} is implemented in Pytorch and all experiments are performed on a NVIDIA V100 GPU with 32 GB memory.
All parameters of the \mymodel{} and logistic regression classifier are optimized by the Adam optimizer~\cite{kingma2014adam}.
For fair evaluation, following previous works~\cite{Zhang2021FromCC,Lee2022AugmentationFreeSL,wang2023singlepass}, we report the average test accuracy with the corresponding standard deviation through 20 random initializations on each dataset. 
A few baselines do not use the public split of Cora, CiteSeer and PubMed, or do not provide results for some of the datasets, we reproduce them based on their codes in the same experimental setup as ours.

We use grid search to find the optimal hyperparameters for each dataset. 
For generic hyperparameters, we find the optimal embedding dimension among $\{128, 256, 512, 1024\}$, the number of GCN layers among $\{1, 2\}$, the activation function of GCN among \{relu, prelu, rrelu\}, the learning rate and weight decay between $[1e^{-8}, 1e^{-2}]$, the temperature $\tau$ among $ \{0.1, 0.2, ..., 0.9\}$, the spacing size of $\tau$ between two adjoining hops among $ \{0, 0.0125, 0.025, 0.05, 0.1 \}$~\cite{hoffmann2022ranking}. 
For hyperparameters specific to \mymodel{}, we conduct grid search over the neighborhood range $k$ among $\{1, 2, 3, 4\}$ and the similarity threshold $\alpha$ in $[0.0001, 1.0]$, which controls the strength of relative similarity constraints as discussed in Section~\ref{sec:learning_framework}.
\section{Additional Results}

\subsection{Impact of Neighborhood Range}
\label{appendix:experiments:neighbor-range}

To investigate the effects of $k$ on the performance of \mymodel{}.
We conduct experiments on Cora, CiteSeer, Photo (\textit{small homophiliy graphs}), WikiCS, Co.CS (\textit{large homophiliy graphs}), Chameleon, Squirrel and Actor (\textit{heterophily graphs}) to investigate the effect of neighborhood range on \mymodel{}\textsubscript{\textsc{Pair}} and \mymodel{}\textsubscript{\textsc{List}} respectively. 
As shown in Figure~\ref{fig:different_hops}, we can observe that the best performance of \mymodel{} is achieved at different $k$-hop neighborhood range due to different characteristics of datasets.

\clearpage
\bibliographystyle{named}
\bibliography{9-ref}

\end{document}